\documentclass{article}

\usepackage[preprint]{neurips_2025}

\usepackage[utf8]{inputenc} 
\usepackage[T1]{fontenc}    
\usepackage{hyperref}       
\usepackage{url}            
\usepackage{booktabs}       
\usepackage{amsfonts}       
\usepackage{nicefrac}       
\usepackage{microtype}      
\usepackage[dvipsnames]{xcolor}         

\usepackage{algorithm}
\usepackage{algorithmic}
\usepackage{graphicx}
\usepackage{amsmath}
\usepackage{amssymb}
\usepackage{mathtools}
\usepackage{amsthm}
\allowdisplaybreaks[1]
\usepackage[capitalize,noabbrev]{cleveref}
\usepackage{bm}
\usepackage{bbm}
\usepackage{kotex}
\usepackage{subfig}
\usepackage[inkscapelatex=false]{svg}
\usepackage{subcaption}
\usepackage{caption}
\usepackage{enumitem}
\usepackage{pifont}
\usepackage{multirow}
\usepackage{array}
\newcolumntype{P}[1]{>{\centering\arraybackslash}p{#1}} 

\usepackage{amsmath,amsfonts,bm}

\def\eqref#1{equation~\ref{#1}}

\def\1{\bm{1}}

\DeclareMathAlphabet{\mathsfit}{\encodingdefault}{\sfdefault}{m}{sl}
\SetMathAlphabet{\mathsfit}{bold}{\encodingdefault}{\sfdefault}{bx}{n}

\theoremstyle{plain}
\newtheorem{theorem}{Theorem}[section]
\newtheorem{proposition}[theorem]{Proposition}

\theoremstyle{definition}
\newtheorem{definition}[theorem]{Definition}

\theoremstyle{remark}

\title{Fair Bayesian Model-Based Clustering}

\author{%
  Jihu Lee\footnotemark[4] \\
  Department of Statistics\\
  Seoul National University\\
  \texttt{rieky0426@snu.ac.kr} \\
  \And
  Kunwoong Kim\footnotemark[4] \\
  Department of Statistics\\
  Seoul National University\\
  \texttt{kwkim.online@gmail.com} \\
  \And
  Yongdai Kim \\
  Department of Statistics\\
  Seoul National University\\
  \texttt{ydkim0903@gmail.com} \\
}

\begin{document}

\maketitle

\renewcommand{\thefootnote}{\fnsymbol{footnote}}
\footnotetext[4]{These authors contributed equally to this work.}

\begin{abstract}
    Fair clustering has become a socially significant task with the advancement of machine learning technologies and the growing demand for trustworthy AI. Group fairness ensures that the proportions of each sensitive group are similar in all clusters.
    Most existing group-fair clustering methods are based on the $K$-means clustering and thus require the distance between instances and the number of clusters to be given in advance.
    To resolve this limitation, we propose a fair Bayesian model-based clustering called Fair Bayesian Clustering (FBC).
    We develop a specially designed prior which puts its mass only on fair clusters, and implement an efficient MCMC algorithm.
    Advantages of FBC are that it can infer the number of clusters and can be applied to any data type as long as the likelihood is defined (e.g., categorical data).
    Experiments on real-world datasets show that FBC (i) reasonably infers the number of clusters, (ii) achieves a competitive utility–fairness trade-off compared to existing fair clustering methods, and (iii) performs well on categorical data. 
\end{abstract}


\section{Introduction}
\label{sec:introduction}

With the rapid development of machine learning-based technologies, algorithmic fairness has been considered as an important social consideration when making machine learning-based decisions.
Among diverse tasks combined with algorithmic fairness, fair clustering \citep{chierichetti2017fair} has received much interest, which ensures that clusters maintain demographic fairness across sensitive attributes such as gender or race.
However, most of existing fair clustering methods can be seen as modifications of $K$-means clustering algorithms, and thus they require (i) the number of clusters and (ii) the distance between two instances given a priori, limiting their adaptability to various kinds of real-world datasets, where the optimal number of clusters is unknown or it is hard to define a reasonable distance (e.g., categorical data).

This paper aims to develop a fair clustering algorithm based on a mixture model with an unknown number of clusters.
To infer fair clusters as well as the number of clusters, we resort to a Bayesian approach where the number of clusters is treated as a parameter to be inferred.
Specifically, we propose a fair Bayesian mixture model with an unknown number of clusters and develop an MCMC algorithm to approximate the posterior distribution of fair clusters as well as the number of clusters.

Among standard (fairness-agnostic) clustering approaches, various Bayesian models have been proposed to infer the number of clusters including mixture models with unknown components \cite{richardson1997bayesian,nobile2007bayesian,mccullagh2008many,miller2018mixture} and mixture of Dirichlet process models \cite{ferguson1973bayesian,antoniak1974mixtures,escobar1995bayesian,neal2000markov}.
For fair Bayesian mixture models, we first modify the standard Bayesian mixture model so that the prior concentrates on fair mixture models (called the fair prior) and so does the posterior distribution.
The basic idea of the fair prior is to put a fair constraint on the cluster assignments.
In general, however, computation of the posterior on a constraint parameter space would be computationally demanding and frequently infeasible. 

To address this challenge, we develop a fair prior based on the idea of matching instances from different sensitive groups.
An important advantage is that our proposed fair prior does not involve any explicit constraint on the parameters and thus posterior approximation by use of an MCMC algorithm can be done without much hamper.
The idea of matching instances has also been successfully implemented for fair supervised learning \cite{kim2025fairness} and non-Bayesian fair clustering \cite{kim2025fairclusteringalignment}.

Our main contributions are summarized as follows:
\begin{itemize}[noitemsep, topsep=0pt]
    \item[$\diamond$] We propose a definition of the fair mixture model. 
    \item[$\diamond$] We develop a novel MCMC algorithm to infer the fair mixture model with an unknown number of clusters, called Fair Bayesian Clustering (FBC).
    \item[$\diamond$] Experimentally, we show that FBC is competitive to existing non-Bayesian fair clustering methods when the number of clusters is given and infers a proper number of clusters well.
\end{itemize}

\section{Related works}
\label{sec:related_works}

\paragraph{Bayesian model-based clustering}

The mixture model is a model-based clustering approach, where each instance (or observation) is assumed to follow a mixture of parametric distributions independently \cite{ouyang2004gaussian,reynolds2009gaussian}.
Popular examples of parametric distributions are the Gaussian \cite{maugis2009variable, yang2012robust, zhang2021gaussian}, the student-t \cite{peel2000robust}, the skew-normal \cite{lin2007finite}, and the categorical \cite{pan2014bayesian,mclachlan2000finite}. 

Bayesian approaches have been popularly used for inference of model-based clustering since they can infer the number of clusters as well as cluster centers.
There are several well-defined Bayesian model-based clustering with unknown number of clusters including Mixture of Finite Mixtures (MFM) \cite{richardson1997bayesian,nobile2007bayesian,mccullagh2008many,miller2018mixture} and Dirichlet Process Mixture (DPM) \cite{ferguson1973bayesian,antoniak1974mixtures,escobar1995bayesian,neal2000markov}.
There also exist various MCMC algorithms for MFM and DPM, such as Reversible Jump Markov Chain Monte Carlo (RJMCMC) \cite{richardson1997bayesian} and Jain-Neal split-merge algorithm \cite{jain2004split}.

\paragraph{Fair clustering}

Given a pre-specified sensitive attribute (e.g., gender or race), the concept of fair clustering is first introduced by \cite{chierichetti2017fair}, with the aim of ensuring that the proportion of each sensitive group within each cluster matches the overall proportion in the entire dataset. 
This fairness criterion is commonly referred to as group fairness, and is also known as proportional fairness. 
Recently, several algorithms have been developed to maximize clustering utility under fairness constraints: 
\cite{chierichetti2017fair, backurs2019scalable} transform training data to fair representations to achieve fairness guarantee prior to clustering, 
\cite{kleindessner2019guarantees, ziko2021variational, li2020deep, zeng2023deep} incorporate fairness constraints or penalties directly during the clustering process, and 
\cite{bera2019fair, NEURIPS2020_a6d259bf} refine cluster assignments while holding pre-determined cluster centers fixed.  
These methods, however, require the number of clusters to be given in advance. 
This limitation, i.e., the lack of fair clustering algorithms that are adaptive to the unknown number of clusters, serves as the motivation for this study.
In addition, model-based clustering can be applied to data for which a meaningful distance is difficult to define (e.g., categorical data).

\section{Fair mixture model for clustering}\label{sec:fairbayesianmodel}

We consider group fairness, as it is one of the most widely studied notions of fairness \citep{chierichetti2017fair, kim2025fairclusteringalignment, backurs2019scalable, ziko2021variational, li2020deep, zeng2023deep, bera2019fair, NEURIPS2020_a6d259bf}.
For simplicity, we only consider a binary sensitive attribute, but provide a method of treating a multinary sensitive attribute in \cref{sec:multiple_s}.
Let $s \in \{ 0, 1 \}$ be a binary sensitive attribute known a priori.
We define two sets of instances (observed data) from the two sensitive groups as $\mathcal{D}^{(s)} = \{ X_i^{(s)} : X_i^{(s)} \in \mathcal{X} \subseteq \mathbb{R}^{d} \}_{i=1}^{n_{s}}$ for $s \in \{0,1\},$ where $\mathcal{X}$ is the support of $X,$ $n_{s}$ is the number of instances from the sensitive group $s,$ and $d$ is the number of features.
Let $\mathcal{D} := \mathcal{D}^{(0)} \cup \mathcal{D}^{(1)}$ be the set of the entire instances.
For a given $K \in \mathbb{N},$ we denote $\mathcal{S}_{K}$ as the $K$-dimensional simplex.

\subsection{Standard mixture model}

The standard finite mixture model \cite{ouyang2004gaussian,reynolds2009gaussian, maugis2009variable, yang2012robust, zhang2021gaussian} without considering fairness is given as
\begin{equation}
    X_{1}, \ldots, X_{n} \overset{\textup{i.i.d.}}{\sim} \sum_{k=1}^{K} \pi_{k} f (\cdot \vert \theta_k)
    \label{eq:standard_2}
\end{equation}
where $n := n_{0} + n_{1}$ and
\begin{equation}\label{eq:concat_X}
    X_{i} :=
    \begin{cases}
        X_{i}^{(0)} & \text{for } i \in \{1, \ldots, n_{0}\} \\
        X_{i-n_{0}}^{(1)} & \text{for } i \in \{n_{0}+1, \ldots, n_{0}+n_{1}\}.
    \end{cases}
\end{equation}
In view of clustering, $K$ is considered to be the number of clusters,
$\bm{\pi} := (\pi_k)_{k=1}^{K} \in \mathcal{S}_{K}$ are the proportions of instances belonging to the $k^{\textup{th}}$ cluster,
and $f(\cdot|\theta_k)$ is the density of instances in the $k^{\textup{th}}$ cluster.

An equivalent representation of the above model can be made by introducing latent variables $Z_{i} \in [K], i \in [n]$ as:
\begin{align}
    & Z_1,\dots,Z_n\overset{\textup{i.i.d.}}{\sim}\textup{Categorical}(\bm{\pi}) \label{eq:standard_with_z-1}
    \\
    & X_i \vert Z_i \sim f (\cdot \vert \theta_{Z_i}) \label{eq:standard_with_z-2}
\end{align}
Note that the latent variable $Z_{i}$ takes the role of the cluster assignment of $X_{i}$ for all $i \in [n].$
That is, when $Z_i=k,$ we say that $X_i$ belongs to the $k^{\textup{th}}$ cluster.

\subsection{Fair mixture model}\label{sec:fair_mixture}

To define a fair mixture model, we first generalize the formulation of the standard finite mixture model in \cref{eq:standard_with_z-1,eq:standard_with_z-2} by considering the dependent latent variables, which we call the \textit{generalized finite mixture model} in this paper.
Let $\bm{Z} := (Z_1,\dots,Z_n).$
Then, the generalized finite mixture model is defined as:
\begin{align}
    & \bm{Z}  \sim G(\cdot)\label{eq:gen_mixture_model_1}
    \\
    & X_i \vert Z_{i} \sim f(\cdot\vert\theta_{Z_i})\label{eq:gen_mixture_model_2}
\end{align}
where $G,$ i.e., the joint distribution of $\bm{Z},$ has the support on $[K]^{n}.$
When $G$ is equal to $\textup{Categorical}(\bm{\pi})^n$ (i.e., $Z_i, \forall i \in [n]$ independently follows $\textup{Categorical}(\bm{\pi})$), the generalized mixture model becomes the standard finite mixture model in \cref{eq:standard_with_z-1,eq:standard_with_z-2}.

We also define the (group) fairness level of $\bm{Z}$:
\begin{equation}\label{eq:def_delta}
    \Delta (\bm{Z}) := \frac{1}{2} \sum_{k=1}^{K} \left\vert \sum_{i=1}^{n_{0}} \mathbb{I}(Z_i^{(0)} = k) / n_{0} - \sum_{j=1}^{n_{1}} \mathbb{I}(Z_j^{(1)} = k) / n_{1} \right\vert
    \in [0, 1]
\end{equation}
where $Z_{i}^{(0)} = Z_{i}$ for $ i \in \{ 1, \ldots, n_0 \} $ and $ Z_{j}^{(1)} = Z_{j + n_0} $ for $ j \in \{ 1, \ldots, n_1 \},$ similar to \cref{eq:concat_X}.
We say that $\bm{Z}$ is fair with fairness level $\epsilon$ if $\bm{Z} \in \mathcal{Z}^{\textup{Fair}}_\epsilon,$
where
\begin{equation}
    \mathcal{Z}^{\textup{Fair}}_\epsilon := \left\{ \bm{Z}: \Delta (\bm{Z}) \le  \epsilon \right\}.
    \label{eq:fair_set}
\end{equation}
In turn, we define a given generalized mixture model to be fair with fairness level $\epsilon$ if the support of $G$ is confined on $\mathcal{Z}^{\textup{Fair}}_\epsilon.$

\subsection{Choice of \texorpdfstring{$G$}{G} for perfect fairness: \texorpdfstring{$\Delta(\bm{Z})=0$}{Delta(Z) = 0}}\label{sec:feas_G}

The art of Bayesian analysis of the fair mixture model is to parameterize $G$ such that posterior inference becomes computationally feasible.
For example, we could consider $\textup{Categorical}(\bm{\pi})^{n}$ (the distribution of independent $Z_i$) restricted to $\mathcal{Z}^{\textup{Fair}}_0$ as a candidate for $G.$
While developing an MCMC algorithm for this distribution would not be impossible, but it would be quite challenging.

In this section, we propose a novel distribution for $G$ in the perfectly fair (i.e., $\Delta (\bm{Z}) = 0$) mixture model with which a practically feasible MCMC algorithm for posterior inference can be implemented.
To explain the main idea of our proposed $G$, we first consider the simplest case of balanced data where $n_{0} = n_{1}$ in \cref{sec:equal_size}.
We then discuss how to handle the case of $n_{0} \neq n_{1}$ in \cref{sec:unequal_size}.

\subsubsection{Case of \texorpdfstring{$n_{0} = n_{1}$}{n0 = n1}}\label{sec:equal_size}
Let $\bar{n} := n_{0} = n_{1}.$
A key observation is that any fair $\bm{Z}$ corresponds to a matching map between $[\bar{n}]$ and $[\bar{n}]$ (an one-to-one map or equivalently a permutation map on $[\bar{n}]$), which is stated in \cref{prop:matching1} with its proof in \cref{sec:theory-omitproof}.
\cref{fig:z} illustrates the relationship between a fair $\bm{Z}$ and a matching map $\mathbf{T} : [\bar{n}] \to [\bar{n}].$
 
\begin{proposition}
    \label{prop:matching1}
    $\bm{Z}\in \mathcal{Z}^{\textup{Fair}}_0$ $\iff$
    There exists a matching map $\mathbf{T}$ such that $Z_{j}^{(1)}=Z^{(0)}_{\mathbf{T}(j)}, \forall j \in [\bar{n}].$ 
\end{proposition}

We utilize the above proposition to define a fair distribution $G.$
The main idea is that \textit{$G$ assigns two matched data to a same cluster once a matching map is given}.
To be more specific, the proposed distribution  is parametrized by $\bm{\pi}$ and $\mathbf{T},$
which is defined by $Z_1^{(0)},\ldots, Z_{\bar{n}}^{(0)} \stackrel{\textup{i.i.d}} \sim \textup{Categorical}(\bm{\pi})$ and $Z_{j}^{(1)}=Z_{\mathbf{T}(j)}^{(0)}.$
It is easy to see that $G$ is perfectly fair.
We write $G(\cdot|\bm{\pi},\mathbf{T})$ for such a distribution.

Denote $\mathcal{Z}^{\textup{Fair}}(\mathbf{T})$ as the support of $G(\cdot|\bm{\pi},\mathbf{T}).$
Note that $\mathcal{Z}^{\textup{Fair}}(\mathbf{T}) \subset \mathcal{Z}^{\textup{Fair}}_0,$ so one might worry that the support of $G(\cdot|\bm{\pi},\mathbf{T})$ is too small.
Proposition \ref{prop:matching1}, however, implies that $\cup_{\mathbf{T}\in \mathcal{T}}\mathcal{Z}^{\textup{Fair}}(\mathbf{T})=\mathcal{Z}^{\textup{Fair}}_0,$ where $\mathcal{T}$ is the set of all matching maps, and thus we can put a prior mass on $\mathcal{Z}^{\textup{Fair}}_0$ by putting a prior on $\mathbf{T}$ (as well as $\bm{\pi}$) accordingly.

\begin{figure}[h]
    \vskip -0.1in
    \centering
    \includegraphics[width=0.42\linewidth]{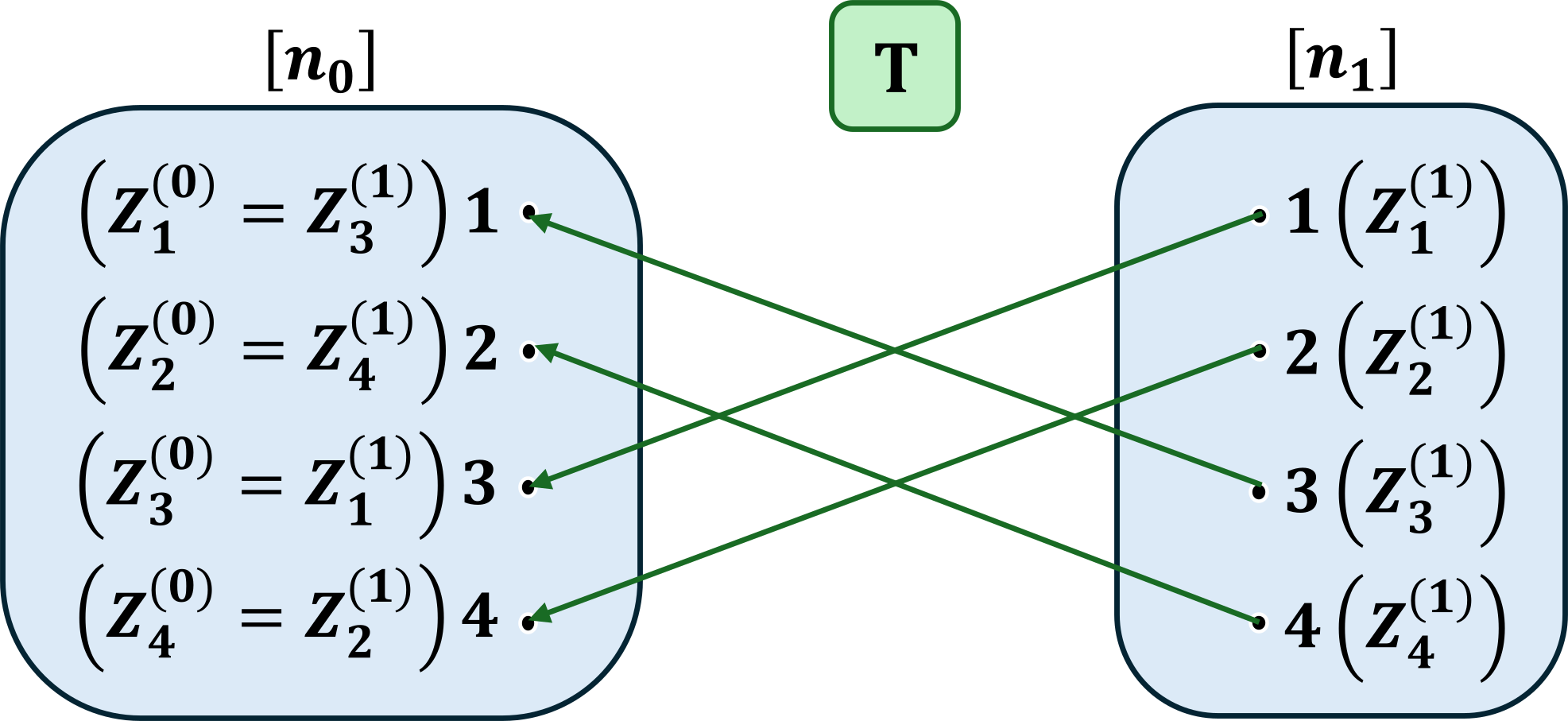}
    \caption{An example illustration of our fair assignment when $\bar{n} = n_{0} = n_{1} = 4,$ where $\mathbf{T}$ is a given matching map such that 
    $ (\mathbf{T}(1), \mathbf{T}(2), \mathbf{T}(3), \mathbf{T}(4)) = (3, 4, 1, 2). $
    }
    \label{fig:z}
    \vskip -0.1in
\end{figure}

A crucial benefit of the proposed $G$ is that $\bm{\pi}$ and $\mathbf{T}$ are not intertwined, hence they can be selected independently.
As a result, we can find a fair mixture model without any additional constraint for the parameters.
It is noteworthy that, conditional on $\mathbf{T},$ the fair mixture model can be written similarly to the standard finite mixture model \cref{eq:gen_mixture_model_1,eq:gen_mixture_model_2} as follows:
\begin{align}
    &(Z^{(0)}_1,\ldots,Z^{(0)}_{\bar{n}})\sim \textup{Categorical}(\bm{\pi})^{\bar{n}} \label{eq:fair_with_z-1}\\
    &X_{\mathbf{T}(j)}^{(0)}, X_{j}^{(1)}\vert Z_{\mathbf{T}(j)}^{(0)}\stackrel{\textup{ind}}{\sim} f(\cdot\vert\theta_{Z_{\mathbf{T}(j)}^{(0)}}) \label{eq:fair_with_z-2}
\end{align}

\subsubsection{Case of \texorpdfstring{$n_{0} \neq n_{1}$}{n0 != n1}}\label{sec:unequal_size}
Without loss of generality, we assume that $n_0<n_1$ such that $n_1=\beta n_0 +r$ for nonnegative integers $\beta$ and $r < n_0.$
We first consider the case of $r=0$ because we can modify $G$ for the balanced case easily.
When $r \ne 0,$ the situation is complicated since there is no one-to-one correspondence between fairness and matching map, and thus we propose a heuristic modification.

\paragraph{Case of \texorpdfstring{$r=0$ :}{r=0}}
A function $\mathbf{T}$ from $[n_1]$ to $[n_0]$ is called a \textit{matching map} if (i) it is onto and (ii) $|\mathbf{T}^{-1}(i)|=\beta$ for all $i\in [n_0].$ 
Let $\mathcal{T}$ be the set of all matching maps.
Then, we have \cref{prop:matching2}, which is similar to \cref{prop:matching1} for the balanced case. 
See \cref{sec:theory-omitproof} for its proof.

\begin{proposition}
\label{prop:matching2}
    $\bm{Z}\in \mathcal{Z}^{\textup{Fair}}_0$ $\iff$ 
    There exists a matching map $\mathbf{T}$ such that $Z_{j}^{(1)}=Z^{(0)}_{\mathbf{T}(j)}, \forall j \in [n_1].$
\end{proposition}

For given $\bm{\pi} \in \mathcal{S}_{K}$ and $\mathbf{T} \in \mathcal{T},$ 
we define a fair distribution $G(\cdot|\bm{\pi}, \mathbf{T})$ as
$Z_1^{(0)}, \ldots, Z_{n_{0}}^{(0)} \stackrel{\textup{i.i.d}} \sim \textup{Categorical}(\bm{\pi})$
and $Z_j^{(1)} = Z_{\mathbf{T}(j)}^{(0)}.$
Similar to the balanced case, we have $\cup_{\mathbf{T}\in \mathcal{T}}\mathcal{Z}^{\textup{Fair}}(\mathbf{T})=\mathcal{Z}^{\textup{Fair}}_0.$ 

\paragraph{Case of \texorpdfstring{$r>0$ :}{r>0}}
It can be shown (see \cref{prop:matching3} in \cref{sec:theory-omitproof} for the proof) that for a given fair $\bm{Z} \in \mathcal{Z}^{\textup{Fair}}_0,$ there exists a function $\mathbf{T}$ from $[n_1]$ to $[n_0]$ such that it is onto, $|\mathbf{T}^{-1}(i)|$ is either $\beta$ or $\beta+1$ and $|R_{\mathbf{T}}|=r,$ where $R_{\mathbf{T}}=\{ i: |\mathbf{T}^{-1}(i)|=\beta+1\}.$ 
Let $\mathcal{T}$ be the set of all such matching maps (functions satisfying these conditions).

A difficulty arises since the converse is not always true.
That is, there exists $\bm{Z}$ that satisfies $Z_{j}^{(1)}=Z_{\mathbf{T}(j)}^{(0)}$ for some $\mathbf{T}\in \mathcal{T}$ but is not fair.
In turn, a sufficient condition of $\bm{Z}$ with $Z_{j}^{(1)}=Z_{\mathbf{T}(j)}^{(0)}$ for some $\mathbf{T}\in \mathcal{T}$ to be fair is that $|C_{k}^{(0)}|/n_0= |C_{k}^{(0)}\cap R_{\mathbf{T}}|/r,$ where $C_{k}^{(0)}=\{i: Z_i^{(0)}=k\}$ for $k\in [K]$ (see \cref{prop:matching4} in \cref{sec:theory-omitproof} for the proof).
That is, fairness of $\bm{Z}$ depends on both $\mathbf{T}$ and $\bm{Z}^{(0)},$ which would make the computation of the posterior inference expensive.

To resolve this difficulty, we propose a heuristic modification of the definition of `fairness of $\bm{Z}$'.
Let $R$ be a subset of $[n_0]$ with $|R|=r,$ and let $\mathcal{T}_{R}$ be a subset of $\mathcal{T}$ such that $R_{\mathbf{T}}=R.$ 
Then, we say that $\bm{Z}$ is fair if there exists $\mathbf{T}\in \mathcal{T}_R$ such that $Z_{j}^{(1)}=Z_{\mathbf{T}(j)}^{(0)}.$ Note that a fair $\bm{Z}$ may not belong to $\mathcal{Z}^{\textup{Fair}}_0$ but the violation of fairness would be small when $|C_{k}^{(0)}|/n_0 \approx |C_{k}^{(0)}\cap R|/r.$
Let $\mathbb{P}_{n_0}$ and $\mathbb{P}_R$ be the empirical distributions of $\{X_i^{(0)}, i\in [n_0] \}$ and $\{X_i^{(0)}, i\in R\},$ respectively.
If $\mathbb{P}_{n_0}(\cdot)=\mathbb{P}_R(\cdot),$ we have $|C_{k}^{(0)}|/n_0 = |C_{k}^{(0)}\cap R|/r,$ and thus any fair $\bm{Z}$ belongs to $\mathcal{Z}^{\textup{Fair}}_0.$
This observation suggests us to choose $R$ such that $\{X_i^{(0)}, i\in R\}$ represents the original data $\{X_i^{(0)}, i\in [n_0]\}$ well. 
In this paper, we consider two candidates for $R:$ 
(i) a random subset of $[n_0]$ and 
(ii) the index of samples closest to the cluster centers obtained by a certain clustering algorithm to $\mathcal{D}^{(0)}$ with $K=r.$
Finally, we can define $G(\cdot|\bm{\pi},\mathbf{T})$ for $\bm{\pi}\in \mathcal{S}_{K}$ and $\mathbf{T}\in \mathcal{T}_R$ similarly to that for the balanced case.

Our numerical studies in \cref{sec:exps-abl} confirm that the two proposed choices of $R$ work quiet well.

\subsection{Choice of \texorpdfstring{$G$}{G} for non-perfect fairness: \texorpdfstring{$\Delta(\bm{Z})>0$}{Delta(Z) > 0}}\label{sec:relax_G}

Once $\bm{\pi}\in \mathcal{S}_{K}$ and $\mathbf{T}\in \mathcal{T}_R$ are given, we can modify $G(\cdot|\bm{\pi},\mathbf{T})$ to have a distribution whose support is included in $ \mathcal{Z}_{\varepsilon}^{\textup{Fair}}.$
The main idea of the proposed modification is to select $m$ many samples from $\mathcal{D}^{(1)}$ and to assign independent clustering labels to them instead of matching them to the corresponding samples in $\mathcal{D}^{(0)}.$
Let $E$ be a subset of $[n_1]$ with $\vert E \vert = m.$
We consider a latent variable $\mathbf{T}_{0} : [n_{1}] \to [n_{0}]$ that is an arbitrary function from $[n_{1}]$ to $[n_{0}].$
Then, we let $Z_j^{(1)}=Z_{\mathbf{T}(j)}^{(0)}$ for $j\in [n_{1}] \setminus E$ and $Z_j^{(1)}=Z_{\mathbf{T}_{0}(j)}^{(0)}$ for $j\in E.$ 
In other words, $\mathcal{D}^{(1)}$ is masked by $E:$
(i) the masked data in $E$ are assigned by $\mathbf{T}_{0}$,
(ii) while the unmasked data in $[n_{1}] \setminus E$ are still assigned by $\mathbf{T}.$
It can be shown that the support of the distribution $G(\cdot|\bm{\pi},\mathbf{T},\mathbf{T}_{0},E)$ belongs to 
$ \mathcal{Z}_{\varepsilon}^{\textup{Fair}} $ with $\varepsilon=m/n_1,$ provided that $r=0.$
See \cref{prop:relaxed_general} in \cref{sec:theory-omitproof} for the theoretical proof.

\section{Fair Bayesian modeling with unknown number of clusters}
\label{sec:proposed_perfect}

In this section, we propose a fair Bayesian mixture model for an unknown number of clusters. 

\subsection{Generative model}\label{sec:method_abfc}

Based on the mechanism in \cref{eq:fair_with_z-1,eq:fair_with_z-2}, we propose a fair mixture model as:
\begin{align}
    & \bm{\pi} \vert K \sim\textup{Dirichlet}_{K} (\gamma,\dots,\gamma), \gamma > 0
    \label{eq:fairmmr_1}
    \\
    & Z_i^{(0)} \vert \bm{\pi} \overset{\textup{i.i.d.}}{\sim}
    \textup{Categorical} (\cdot \vert \bm{\pi}), \quad\forall i \in [n_{0}]
    \label{eq:fairmmr_2}
    \\
    &
    Z_{j}^{(1)} =
    \begin{cases}\label{eq:fairmmr_3}
        Z_{\mathbf{T}(j)}^{(0)}, & \forall j \in [n_{1}]\setminus E
        \\
        Z_{\mathbf{T}_{0}(j)}^{(0)}, & \forall j\in E
    \end{cases}
    \\
    & X_i^{(0)} \vert Z_{i}^{(0)} \sim f(\cdot \vert \theta_{Z_i^{(0)}}),
    X_j^{(1)} \vert Z_{j}^{(1)} \sim f(\cdot \vert \theta_{Z_{j}^{(1)}})
    \label{eq:fairmmr_4}
\end{align}

In this model, $K,\bm{\theta}=(\theta_1,\theta_2,\ldots),E,\mathbf{T}$ and $\mathbf{T}_0$ are the parameters to be inferred.

\subsection{Prior}\label{sec:original_priors}

As a prior, we assume that each parameter is independent: $K\sim p_K(\cdot), E \sim p_E(\cdot),\mathbf{T}\sim p_\mathbf{T}(\cdot), \mathbf{T}_{0} \sim p_{\mathbf{T}_{0}}(\cdot),$ and $\theta_{1}, \theta_2\ldots, \overset{\textup{ind}}{\sim} H.$

\paragraph{Prior for \texorpdfstring{$K$}{K}}
For $K,$ we consider $K \sim p_K (\cdot)$, where $p_K$ is a probability mass function on $\{1,2,\dots\}.$
In this study, we use $\textup{Geometric}(\kappa), \kappa\in (0,1)$ for $p_K.$

\paragraph{Prior for \texorpdfstring{$\bm{\theta}$}{theta}}
Usually, we choose a conjugate distribution of $f(\cdot|\theta)$ for $H.$ 
For example, when $f(\cdot|\theta)$ is the density of the Gaussian distribution,
and $\theta$ consists of the mean vector and diagonal covariance matrix as $\theta= (\bm{\mu},\bm{\lambda}^{-1})\in\mathbb{R}^{d}\times\mathbb{R}^{d},$ where $\bm{\mu} = (\mu_{j})_{j=1}^{d}$ and $\bm{\lambda} = (\lambda_{j})_{j=1}^{d}.$
We let $(\mu_j,\lambda_j), j\in [d]$ be independent and follow
$\lambda_{j} \sim \textup{Gamma}(a, b)$ for some $a, b$ and $\mu_{j} | \lambda_{j} \sim \mathcal{N}(0,\lambda_{j}^{-1}), j \in [d]$ for $H.$

\paragraph{Prior for \texorpdfstring{$\mathbf{T}$}{T}}
Motivated by \citep{NIPS2012_4c27cea8},
we construct a  prior of $\mathbf{T}$ with its support $\mathcal{T}_{R}$ based on the energy defined in \cref{def:energy_T}.
The energy of a random matching map $\mathbf{T}$ is defined as below, which measures the similarity between two matched data.
\begin{definition}[Energy of a matching map]\label{def:energy_T}
    Let $D: \mathcal{X} \times \mathcal{X} \rightarrow \mathbb{R}_{+}$ be a given distance and $\Pi_{n} := \{ \mathbf{T} : [n_{1}] \rightarrow [n_{0}] \}.$
    For a given matching map $\mathbf{T} \in \Pi_{n},$ the energy of $\mathbf{T}$ is defined by
    $
    \mathbf{e}(\mathbf{T}) = \mathbf{e}(\mathbf{T}; \tau) := \exp \left( - \sum_{j=1}^{n_{1}} D \left( X_{\mathbf{T}(j)}^{(0)}, X_{j}^{(1)} \right) / n_{1} \tau \right),
    $
    where $\tau > 0$ is a pre-specified temperature constant.
\end{definition}
See \cref{sec:exps-appen_impl} for the choice of $D$.
For the prior, we let $p_\mathbf{T}(\mathbf{T}) \propto \mathbf{e}(\mathbf{T}) \mathbb{I}(\mathbf{T}\in \mathcal{T}_R).$

\paragraph{Prior for \texorpdfstring{$\mathbf{T}_{0}$}{T0}}
We use the uniform distribution on $\Pi_{n}$.

\paragraph{Prior for \texorpdfstring{$E$}{E}}
We use the uniform distribution on $[n_1:m],$
where $[n_1:m]$ is the collection of all subsets of $[n_1]$ whose cardinality is $m.$

\subsection{Equivalent representation}
\label{sec:equiv_rep}

As is done in \cite{miller2018mixture}, we use the following equivalent representation of the proposed fair Bayesian mixture model. 

\paragraph{Generative model}
For a given partition $\mathcal{C}$ of $[n_{0}],$ an equivalent generative model to the proposed fair mixture model in \cref{eq:fairmmr_1,eq:fairmmr_2,eq:fairmmr_3,eq:fairmmr_4} is: 
\begin{align}
    & \phi_c\overset{\textup{i.i.d.}}{\sim}H, c\in\mathcal{C} \label{eq:eq_rep_1}
    \\
    & X_i^{(0)} \overset{\textup{ind}}{\sim} f(\cdot \vert \phi_c) \quad i \in c \label{eq:eq_rep_2}
    \\
    & X_j^{(1)} \overset{\textup{ind}}{\sim} \label{eq:eq_rep_3}
    \begin{cases}
        f(\cdot \vert \phi_c) & \forall j \in [n_{1}] \setminus E \textup{ s.t. } \mathbf{T}(j) \in c
        \\
        f(\cdot \vert \phi_c) & \forall j \in E \textup{ s.t. } \mathbf{T}_{0}(j) \in c.
    \end{cases}
\end{align}

\paragraph{Prior for \texorpdfstring{$\mathcal{C}$}{C}}
The prior for $\mathcal{C}$ in this equivalent representation is 
$ p_{\mathcal{C}}(\mathcal{C} \vert \mathbf{T},\mathbf{T}_0,E) = p_{\mathcal{C}}(\mathcal{C}) = V_{n_0}(t)\prod_{c\in\mathcal{C}}\gamma^{(|c|)},$ where $t=|\mathcal{C}|,$
$ V_{n_0}(t)=\sum_{k=1}^\infty\frac{k_{(t)}}{(\gamma k)^{({n_0})}}p_K(k),$
$ (\gamma k)^{({n_0})} = (\gamma k + {n_0} - 1)! / (\gamma k - 1)!$
and $k_{(t)} = k! / (k-t)!.$
See \cite{miller2018mixture} for the derivation. 

\paragraph{Priors for \texorpdfstring{$\mathbf{T}, \mathbf{T}_{0}$ and $E$}{T, T0 and E}}
The prior of $(\mathbf{T},\mathbf{T}_{0},E)$ remains the same as the prior introduced in \cref{sec:original_priors}.

\section{Inference algorithm: FBC}\label{sec:inference}
We develop an MCMC algorithm for the proposed fair Bayesian mixture model in \cref{sec:equiv_rep}. 
Here, we denote $\Phi$ as the mixture parameters to be sampled (i.e., $\Phi = \mathcal{C}$ when $H$ is conjugate, or $\Phi = (\mathcal{C}, \bm{\phi})$ when $H$ is non-conjugate, where $\bm{\phi} := \left(\phi_c:c\in\mathcal{C}\right)$).

The posterior sampling of $(\mathbf{T},\mathbf{T}_{0},E,\Phi)\sim p(\mathbf{T},\mathbf{T}_{0},E,\Phi\vert\mathcal{D})$ is done by a Gibbs sampler: 
(i) sampling $(\mathbf{T},\mathbf{T}_{0},E)\sim p(\mathbf{T},\mathbf{T}_{0},E|\Phi,\mathcal{D})$,
and
(ii) sampling $\Phi\sim p(\Phi|\mathbf{T},\mathbf{T}_{0},E,\mathcal{D}).$
We name the proposed MCMC inference algorithm as \textbf{Fair Bayesian Clustering (FBC)}. 
In the subsequent two subsections, we explain how to sample $(\mathbf{T},\mathbf{T}_{0},E)$ and $\Phi$ from their conditional posteriors, respectively, by using a Metropolis-Hasting (MH) algorithm.
We also discuss the extension of FBC for handling a multinary sensitive attribute in \cref{sec:multiple_s}, with experiments on a real dataset.

\subsection{STEP 1 \texorpdfstring{$\triangleright$ Sampling $(\mathbf{T},\mathbf{T}_{0},E)\sim p(\mathbf{T},\mathbf{T}_{0},E|\Phi,\mathcal{D})$}{: Sampling T, T0 and E}}\label{sec:alg-step1}

\paragraph{Proposal}
For the proposal distribution of $(\mathbf{T}',\mathbf{T}_{0}',E')$ from $(\mathbf{T},\mathbf{T}_{0},E),$
we first assume that $\mathbf{T}\rightarrow \mathbf{T}', \mathbf{T}_{0} \rightarrow \mathbf{T}_{0}'$
and $E\rightarrow E'$ are independent.

For the proposal of $\mathbf{T}\rightarrow \mathbf{T}' ,$ we randomly select two indices $i_1$ and $i_2$
from $[n_{1}],$ and define as:
\begin{equation}\label{eq:proposal_T_1}
    \mathbf{T}'(j):=
    \begin{cases}
        \mathbf{T}(j) & \text{for } j \notin [n_{1}] \setminus \{ i_{1}, i_{2} \}
        \\
        \mathbf{T}(i_{2}) & \text{for } j = i_{1}
        \\
        \mathbf{T}(i_{1}) & \text{for } j = i_{2}.
    \end{cases}
\end{equation}
We swap only two indices to guarantee $\mathbf{T}'\in \mathcal{T}_R.$

For $\mathbf{T}_{0}\rightarrow \mathbf{T}_{0}' ,$ we randomly select an index $ j' \in [n_{1}], $ then set $\mathbf{T}_{0}'(j') = i$ where $i \sim \textup{Unif}([n_{0}])$ and $\mathbf{T}_{0}'(j) = \mathbf{T}_{0}(j)$ for $j \ne j'.$ For $E \rightarrow E',$ we randomly swap two indices, one from $E$ and the other from $[n_{1}] \setminus E.$

\paragraph{Acceptance / Rejection}

As the randomness in  the proposal of $\mathbf{T}', \mathbf{T}_{0}'$ and $E'$
does not depend on $\mathbf{T}, \mathbf{T}_{0}$ and $E,$ 
the proposal density ratio $q((\mathbf{T}',\mathbf{T}_{0}',E') \to (\mathbf{T},\mathbf{T}_{0},E)) / q((\mathbf{T},\mathbf{T}_{0},E) \to (\mathbf{T}',\mathbf{T}_{0}',E'))$ is equal to 1.
Hence, the acceptance probability of a proposal $(\mathbf{T}',\mathbf{T}_{0}',E')$ is given as
\begin{equation}
    \begin{aligned}
        \alpha(\mathbf{T}',\mathbf{T}_{0}',E')=\min \left\{1, \frac{e(\mathbf{T}')\mathcal{L}(\mathcal{D};\Phi,\mathbf{T}',\mathbf{T}_{0}',E')}{e(\mathbf{T})\mathcal{L}(\mathcal{D};\Phi,\mathbf{T},\mathbf{T}_{0},E)}\right\}.
    \end{aligned}
    \label{eq:acceptance_ratio}
\end{equation}
See \cref{sec:accept_prob-appen} for the calculation details of the acceptance probability.

We repeat the MH sampling of $(\mathbf{T}', \mathbf{T}_{0}', E')$ multiple times before sampling $\Phi,$ which helps accelerate convergence.
In our experiments, we perform this repetition 10 times.
This additional computation is minimal, even for large datasets, since the acceptance probability requires calculating the likelihood only for instances whose assigned clusters change.
The maximum number of such instances with changed clusters is 4 (2 for $\mathbf{T}',$ 1 for $\mathbf{T}_0'$ and 1 for $E'$).

\subsection{STEP 2 \texorpdfstring{$\triangleright$ Sampling $\Phi\sim p(\Phi \vert \mathbf{T},\mathbf{T}_{0}, E, \mathcal{D})$}{: Sampling mixture parameters}}\label{sec:alg-step2}
Sampling from $p(\Phi\vert \mathbf{T},\mathbf{T}_{0},E,\mathcal{D})$ can be done similar to the sampling 
algorithm for the standard mixture model when $(\mathbf{T},\mathbf{T}_{0},E)$ is given.
Specifically, we mimic the procedure of \cite{miller2018mixture}, which utilizes DPM inference algorithms from \cite{neal2000markov,maceachern1998estimating}.
See \cref{sec:details_STEP2s} for details of this sampling step.

\section{Experiments}\label{sec:exps}

This section presents results of numerical experiments.  
Through analysis on multiple benchmark datasets, we show that
(i) FBC controls the fairness level well and it is compared favorably with existing baselines in terms of the trade-off between clustering utility and fairness level;
(ii) FBC infers the number of clusters $K$ reasonably well;
(iii) FBC is easily applicable to categorical data.

\subsection{Settings}\label{sec:exps-settings}

\paragraph{Datasets and performance measures}

We analyze a synthetically generated 2D toy dataset and three real datasets: \textsc{Diabetes} \citep{pima_diabetes}, \textsc{Adult} \citep{misc_adult_2}, and \textsc{Bank} \citep{MORO201422}.
The toy dataset is sampled from a 6-component Gaussian mixtures.
All features in the datasets are continuous and so we use the fair Gaussian mixture model.
We scale all features of data to have zero mean and unit variance.
See \cref{sec:exps-appen_data} for details about the datasets.

The following are the measures we use for evaluating a given clustering result.
For clustering utility, we consider the \texttt{Cost} (i.e., the average distance to the center of the assigned cluster from each data point) for evaluating the clustering utility, which is defined as:
$ \texttt{Cost} := \sum_{k=1}^{K} \sum_{X_{i} \in C_{k}} \Vert X_{i} - \hat{\mu}_{k} \Vert^{2} / n, $
where $ C_{k} := C_{k}^{(0)} \cup C_{k}^{(1)} $ is the set of data assigned to the $k^{\textup{th}}$ cluster.
Here, $\hat{\mu}_{k} := \sum_{X_{i} \in C_{k}} X_{i} / \vert C_{k} \vert$ denotes the center of the $k^{\textup{th}}$ cluster.

For fairness level, we consider two measures:
(i) the fairness level $\Delta(\bm{Z})$ defined in \cref{sec:fair_mixture} and 
(ii) the balance (\texttt{Bal}) defined as 
$\texttt{Bal} := \min_{k \in [K]} \texttt{Bal}_{k},$ where $\texttt{Bal}_{k} := \min \left\{ \vert C_{k}^{(0)} \vert / \vert C_{k}^{(1)} \vert, \vert C_{k}^{(1)} \vert / \vert C_{k}^{(0)} \vert \right\}$
which is popularly considered in recent fair clustering literature \citep{backurs2019scalable, ziko2021variational, li2020deep, zeng2023deep, bera2019fair, esmaeili2021fair}.
We abbreviate $\Delta(\bm{Z})$ by $\Delta$ in this section.
See \cref{sec:multiple_s} for the definitions of $\Delta$ and \texttt{Bal} for a multinary sensitive attribute.

\paragraph{Set-up of experiments}
For a fairness-agnostic method, we consider the Mixtures of Finite Mixtures (MFM) algorithm proposed by \cite{miller2018mixture}.
For fairness-aware methods, we consider three algorithms: two existing non-Bayesian approaches and a Bayesian approach.
Two non-Bayesian approaches are SFC \cite{backurs2019scalable} and VFC \cite{ziko2021variational}, the first of which is a fairlet-based approach and the other is an in-processing approach by adding a fairness regularizer.
For the Bayesian approach, we use the FBC while the matching map between the two sensitive groups is fixed in advance by a fairlet \cite{chierichetti2017fair,backurs2019scalable}, which we call Fair MFM (i.e., MFM on a pre-specified fairlet space).
For FBC, we run 1200 iterations, discard the first 1000 as burn-in, and collect 200 post-burn-in samples.
The main results reported in this section are based on a sample randomly selected from these 200 post–burn-in samples.
Whenever $r > 0$, we set $R$ (defined in \cref{sec:unequal_size}) as a random subset of $[n_0]$ of size $r$.
Additional experimental details, including prior parameter selection and implementations of the baseline methods, are provided in \cref{sec:exps-appen_impl}.

\subsection{Simulation on a toy dataset}\label{sec:toy}

\cref{fig:toy} compares how  MFM and FBC behave differently on the toy dataset.
See \cref{sec:exps-appen_data} for the generative model of the toy data.
Note that MFM is perfectly unfair while FBC is perfectly fair.
Moreover, we can see that enforcing fairness significantly affects not only cluster locations but also the number of clusters.

\begin{figure*}[h]
    \vskip -0.15in
    \centering
    \subfloat{\includegraphics[width=0.4\linewidth]{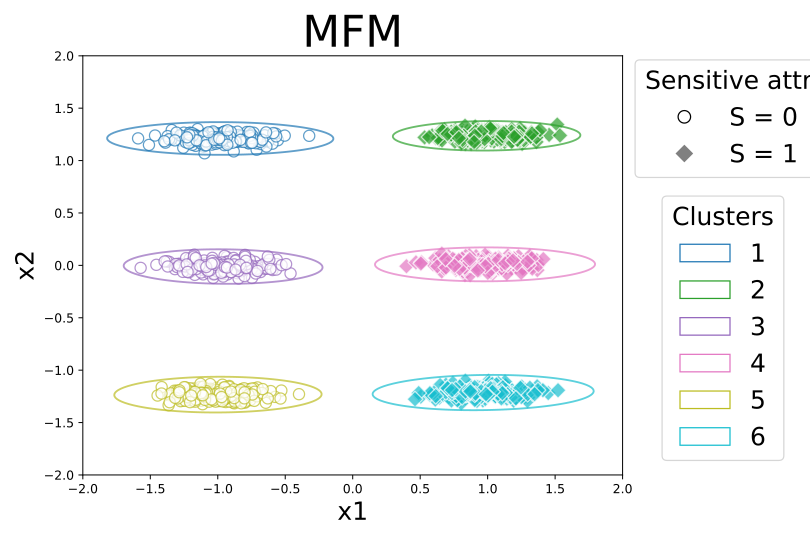}}
    \subfloat{\includegraphics[width=0.4\linewidth]{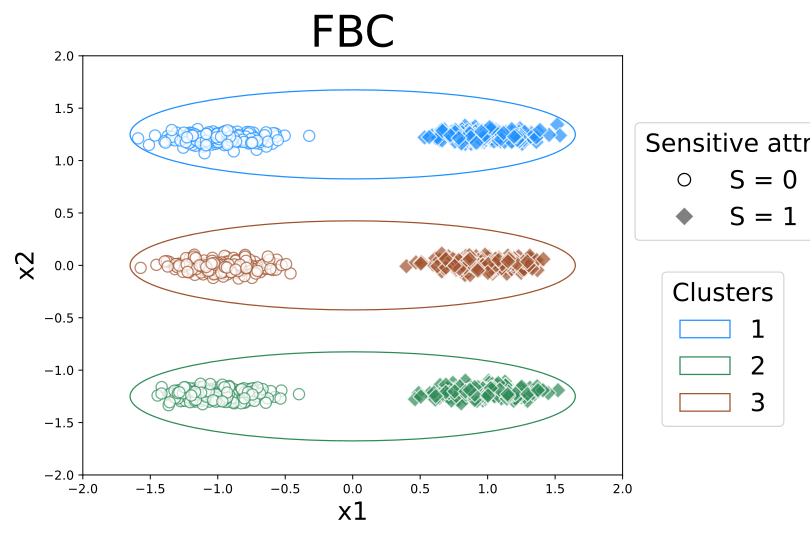}}
    \caption{
    Visualization of the clustering results on a toy dataset (left - MFM with $K = 6$, right - FBC with $K = 3$).
    The shape represents the sensitive attribute (i.e., {\Large $\circ$} for $S = 0$ and \textcolor{gray}{$\blacklozenge$} for $S = 1$).
    The colored ellipses correspond to clusters made.
    }
    \vskip -0.15in
    \label{fig:toy}
\end{figure*}

\subsection{Fair clustering performance}\label{sec:exp-compare}

\paragraph{Comparison of FBC to existing fair clustering algorithms}

We investigate whether FBC provides reasonable clustering results, by assessing the trade-off between utility (\texttt{Cost}) and fairness ($\texttt{Bal}$ and $\Delta$).
Specifically, we compare FBC (with $\Delta \approx 0$ by setting $m = 0$) with the three baselines (SFC, VFC, and Fair MFM).
For SFC and VFC, we set $K$ as the inferred one by FBC for a fair comparison, while Fair MFM infers $K$ by itself.
\cref{table:main-compare} shows that FBC achieves the lowest $\Delta$ without hampering \texttt{Cost} much, when compared to SFC, VFC, and Fair MFM.

\begin{table}[h]
    \vskip -0.1in
    \footnotesize
    \caption{
    Comparison of \texttt{Cost}, $\Delta,$ and $\texttt{Bal}$ on the three real datasets.
    The $K$s for SFC and VFC are fixed at the inferred $K$ by FBC.
    \textbf{Bold}-faced results indicate the lowest $\Delta$s.
    }
    \label{table:main-compare}
    \centering
    \vskip 0.1in
    \begin{tabular}{c||P{0.2cm}P{0.8cm}P{0.7cm}P{0.65cm}|P{0.2cm}P{0.8cm}P{0.7cm}P{0.65cm}|P{0.2cm}P{0.85cm}P{0.7cm}P{0.65cm}}
        \toprule
        Dataset & \multicolumn{4}{c|}{\textsc{Diabetes}} & \multicolumn{4}{c|}{\textsc{Adult}} & \multicolumn{4}{c}{\textsc{Bank}}
        \\
        \midrule
        \multirow{2}{*}{Method} & $K$ & \texttt{Cost} ($\downarrow$) & $\Delta$ ($\downarrow$)& $\texttt{Bal}$ ($\uparrow$) & $K$ & \texttt{Cost} ($\downarrow$) & $\Delta$ ($\downarrow$) & $\texttt{Bal}$ ($\uparrow$) & $K$ & \texttt{Cost} ($\downarrow$) & $\Delta$ ($\downarrow$) & $\texttt{Bal}$ ($\uparrow$)
        \\
        \midrule
        SFC & 4 & 6.789 & 0.026 & 0.881 & 3 & 6.774 & 0.011 & 0.438 & 7 & 5.720 & 0.050 & {0.650}
        \\
        VFC & 4 & 6.761 & 0.039 & 0.858 & 3 & 5.763 & 0.013 & 0.423 & 7 & 4.643 & 0.052 & 0.524
        \\
        Fair MFM & 4 & {6.337} & 0.026 & {0.922} & 2 & 4.970 & {0.014} & {0.440} & 2 & 6.081 & 0.121 & 0.533
        \\
        FBC $\checkmark$ & 4 & 6.795 & \textbf{0.012} & {0.910} & 3 & {4.990} & \textbf{0.006} & 0.429 & 7 & 5.947 & \textbf{0.020} & {0.500}
        \\
        \bottomrule
    \end{tabular}
\end{table}

\paragraph{Fairness level control of FBC}

We also confirm that the fairness level can be controlled well by controlling $m$ (the size of $E$).
\cref{fig:control} in \cref{sec:exps-appen} presents the trade-off between $m$ and the fairness levels $\Delta$ and \texttt{Bal}, showing that a smaller $m$ usually results in a fairer clustering.

\paragraph{Analysis with a multinary sensitive attribute}
See \cref{sec:exp-compare_appen} for the analysis on \textsc{Bank} dataset with three sensitive groups, which shows that FBC still works well for a multinary sensitive attribute.

\subsection{Inference of the number of clusters \texorpdfstring{$K$}{K}}\label{sec:exps-inferK}

\cref{fig:t_posterior_plot} draws the posterior distributions of $K$ for the three datasets.
The posterior distributions are well-concentrated around the posterior modes.
See \cref{fig:K_nll} in \cref{sec:appen-exp-K_nll} for an evidence that the posterior modes are good choices for the number of clusters.

\begin{figure}[h]
    \vskip -0.1in
    \centering
    \includegraphics[width=.28\linewidth]{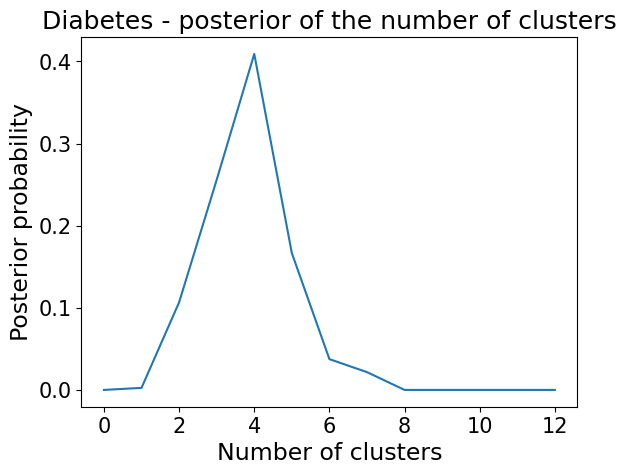}
    \includegraphics[width=.27\linewidth]{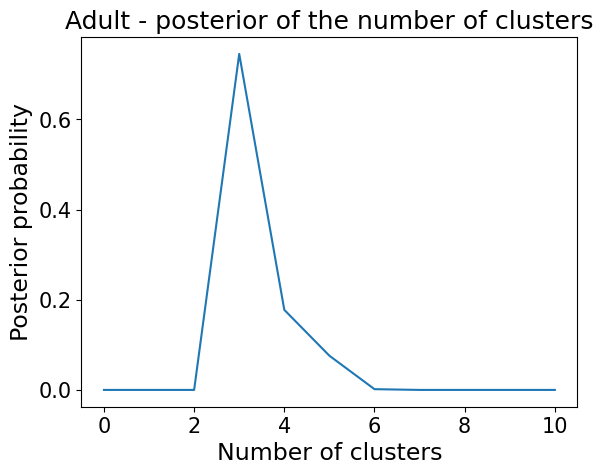}
    \includegraphics[width=.27\linewidth]{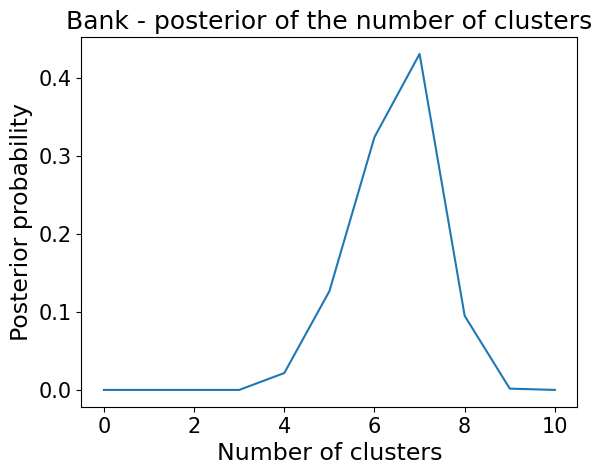}
    \caption{
    Posteriors of $K$ on (left) \textsc{Diabetes}, (center) \textsc{Adult}, and (right) \textsc{Bank} datasets.
    }
    \label{fig:t_posterior_plot}
    \vskip -0.1in
\end{figure}

\subsection{Application of FBC to categorical data}\label{sec:exp-categorical}

To explore how FBC works with categorical data, we synthetically build a categorized \textsc{Bank} dataset, where all six continuous features are binarized by their median values. 
We then apply FBC to this dataset and observe that it provides a significantly fairer clustering compared to the fairness-agnostic method (MFM). 
See \cref{sec:appen-exp-categorical} for details on the dataset construction and how we apply FBC to this dataset, where the results are presented in \cref{table:mixed}.

\subsection{Ablation studies}\label{sec:exps-abl}

\paragraph{Temperature \texorpdfstring{$\tau$}{tau}}
We analyze the impact of the temperature constant $\tau$ in the prior of $\mathbf{T}$ defined in \cref{def:energy_T}.
To do so, we set $E = \emptyset$ and compare the utility (\texttt{Cost}) and fairness levels ($\texttt{Bal}$ and $\Delta$) for various values of $\tau.$ 
\cref{table:appen-tau} in \cref{sec:appen-exps-abl} shows that FBC is not sensitive to the choice of $\tau.$

\paragraph{Choice of \texorpdfstring{$R$ when $r>0$}{R when r>0}}

We investigate the validity of the two proposed heuristic approaches of selecting $R$ in \cref{sec:unequal_size}.
See \cref{table:appen-heuristic} in \cref{sec:appen-exps-abl} for the results, which suggest that the two approaches (random vs. cluster centers) are not much different. 
Since the use of cluster centers requires additional computation, we recommend using a random subset of $[n_{0}]$ for $R.$

\paragraph{Prior parameter \texorpdfstring{$\kappa$ in $p_{K}$}{kappa in p_K}}

We use $\kappa = 0.1$ in $\textup{Geometric}(\kappa)$ for our experiments.
To investigate sensitivity, we vary $\kappa \in \{0.1, 0.3, 0.5, 0.7, 0.9\}$ and observe that FBC is not much sensitive to $\kappa$ (see \cref{fig:kappa_abl} in \cref{sec:appen-exps-abl}).

\paragraph{Convergence of MCMC}

We assess the convergence of the MCMC algorithm by monitoring (i) the inferred $K$ and (ii) the negative log-likelihood ($\texttt{NLL}$) on training data, over sampling iterations.
As is done by \cite{miller2018mixture}, the results in \cref{fig:conv_abl,fig:conv_abl_nll} of \cref{sec:appen-exps-abl} show that the MCMC algorithm converges well and finds good clusters.

\section{Concluding remarks}\label{sec:conc}

In this paper, we proposed a fair Bayesian mixture model for fair clustering with unknown number of clusters and developed an efficient MCMC algorithm called FBC.
A key advantage of FBC is its ability to infer the number of clusters.
In addition, the proposed MCMC algorithm can be further accelerated by adding the Jain-Neal split-merge sampler of \cite{jain2004split} for sampling $\Phi.$

A problem not pursued in this work is the assignment of new data to the learned clusters. 
In standard mixture models, a new datum can be assigned to a cluster based on the posterior distribution.
In contrast, FBC requires matched instances to be assigned to the same cluster, but the inferred matching map is only available for the training data.  
To address this, one may approximate the learned matching map using a parametric model and apply it to new test data.  
We leave this approach for future work.

From a societal perspective, by mitigating unfair bias between privileged and underprivileged groups in clustering, our approach may contribute to more trustworthy decision-making in real-world applications.

\clearpage
\bibliographystyle{unsrt}
\bibliography{references.bib}

\clearpage
\appendix

\section{Theoretical results for \texorpdfstring{\cref{sec:fairbayesianmodel}}{Section 3}}\label{sec:theory-omitproof}


\textbf{\cref{prop:matching1}}
    Assume that $n_{0} = n_{1} = \bar{n}.$
    Then, we have:
    $\bm{Z}\in \mathcal{Z}^{\textup{Fair}}_0$ $\iff$
    There exists a matching map $\mathbf{T}$ such that $Z_{j}^{(1)}=Z^{(0)}_{\mathbf{T}(j)}, \forall j \in [\bar{n}].$ 

\begin{proof}[Proof of \cref{prop:matching1}]
    See the proof of \cref{prop:matching2}, as this proposition is a special case of \cref{prop:matching2}, with $\beta = 1.$
\end{proof}

\textbf{\cref{prop:matching2}}
    Assume that $n_{1} = \beta n_{0}$ for some positive integer $\beta.$
    Then, we have:
    $\bm{Z}\in \mathcal{Z}^{\textup{Fair}}_0$ $\iff$ 
    There exists a matching map $\mathbf{T}$ such that $Z_{j}^{(1)}=Z^{(0)}_{\mathbf{T}(j)}, \forall j \in [n_1].$

\begin{proof}[Proof of \cref{prop:matching2}]
    ($\implies$)
    Recall that $C_{k}^{(0)} = \{ i : Z_{i}^{(0)} = k \}$ and $C_{k}^{(1)} = \{ j : Z_{j}^{(1)} = k \}.$
    Note that $\bm{Z} \in \mathcal{Z}^{\textup{Fair}}_0$ implies $\beta \vert C_{k}^{(0)} \vert = \vert C_{k}^{(1)} \vert$ for all $k \in [K].$
    Hence, for all $k \in [K],$ there exists an onto map $\mathbf{T}_{k}$ from $C_{k}^{(1)}$ to $C_{k}^{(0)}$ such that $\vert \mathbf{T}_{k}^{-1}(i) \vert = \beta$ for all $i \in C_{k}^{(0)}.$
    Letting $\mathbf{T}(j) := \sum_{k=1}^{K} \mathbf{T}_{k}(j) \mathbb{I}(j \in C_{k}^{(1)}), j \in [n_{1}]$ concludes the proof.

    ($\impliedby$)
    Suppose that there exists a matching map $\mathbf{T}$ such that $Z_{j}^{(1)}=Z^{(0)}_{\mathbf{T}(j)}$ for all $j \in [n_{1}].$
    Then, we have $ \beta \vert C_{k}^{(0)} \vert = \vert C_{k}^{(1)} \vert$ for all $k \in [K].$
    Hence, $\sum_{i=1}^{n_{0}} \mathbb{I}(Z_{i}^{(0)} = k) / n_{0} 
    = \vert C_{k}^{(0)} \vert / n_{0} 
    = \beta \vert C_{k}^{(0)} \vert / n_{1} 
    = \vert C_{k}^{(1)} \vert / n_{1} 
    = \sum_{j=1}^{n_{1}} \mathbb{I}(Z_{j}^{(1)} = k) / n_{1},$ which implies $\bm{Z} \in \mathcal{Z}^{\textup{Fair}}_0.$
\end{proof}

\cref{prop:matching3,prop:matching4} below support the claims in \cref{sec:unequal_size}.
For enhanced readability, we first recall some notations/assumptions:
\begin{itemize}
    \item $n_{1} = \beta n_{0} + r.$
    
    \item $\mathbf{T} : [n_1] \to [n_0]$ is a function such that 
    it is onto, $|\mathbf{T}^{-1}(i)|$ is either $\beta$ or $\beta+1$
    and $|R_{\mathbf{T}}|=r,$ where
    $R_{\mathbf{T}}=\{ i: |\mathbf{T}^{-1}(i)|=\beta+1\}.$ 
    Let $\mathcal{T}$ be the set of all such matching maps (functions satisfying these conditions).

    \item Let $R_{\mathbf{T}}=\{ i: |\mathbf{T}^{-1}(i)|=\beta+1\}$ for a given $\mathbf{T}: [n_{1}] \to [n_{0}].$

    \item For given $Z_{i}^{(0)}, i \in [n_{0}]$ and $Z_{j}^{(1)}, j \in [n_{1}],$
    we define
    $C_{k}^{(0)}=\{\,i:Z_{i}^{(0)}=k\}$
    and
    $C_{k}^{(1)}=\{\,j:Z_{j}^{(1)}=k\}$
    for $k \in [K].$
\end{itemize}

\begin{proposition}\label{prop:matching3}
    Assume that $r > 0.$
    For a given fair $\bm{Z} \in \mathcal{Z}^{\textup{Fair}}_0,$ there exists a function $\mathbf{T}\in \mathcal{T}.$
\end{proposition}

\begin{proof}[Proof of \cref{prop:matching3}]
    Since $\bm{Z} \in \mathcal{Z}^{\textup{Fair}},$ for all $k \in [K]$ we have
    $|C_{k}^{(1)}| = \frac{n_{1}}{n_{0}}\,|C_{k}^{(0)}|,$ which is an integer.
    Let 
    $k^{*}=\arg\min_{k} |C_{k}^{(0)}|$
    and
    $\alpha=\frac{r\,|C_{k^{*}}^{(0)}|}{n_{0}}\in\mathbb{N}.$
    Then, we can construct
    $\mathbf{T}_{k^{*}}\colon C_{k^{*}}^{(1)}\to C_{k^{*}}^{(0)}$
    so that exactly $\alpha$ elements of $C_{k^{*}}^{(0)}$ have preimage size $\beta+1$, and the remaining $|C_{k^{*}}^{(0)}|-\alpha$ have size $\beta$.
    
    Next, for each $l\neq k^{*}$, set 
    $a_{l}=\frac{|C_{l}^{(0)}|}{|C_{k^{*}}^{(0)}|}\in\mathbb{N},$
    so that $|C_{l}^{(0)}|=a_{l} |C_{k^{*}}^{(0)}|$ and $|C_{l}^{(1)}| = a_{l} |C_{k^{*}}^{(1)}|$.  Define 
    $\alpha_{l}=a_{l}\,\alpha,$
    and similarly construct
    $\mathbf{T}_{l}\colon C_{l}^{(1)}\to C_{l}^{(0)}$
    so that exactly $\alpha_{l}$ points in $C_{l}^{(0)}$ have preimage size $\beta+1$, and the remaining $|C_{l}^{(0)}|-\alpha_{l}$ have size $\beta$.
    
    Finally, let $\mathbf{T}(j) := \sum_{k=1}^{K} \mathbf{T}_{k}(j) \mathbb{I}(j \in C_{k}^{(1)}), j \in [n_{1}].$
    Then, $\mathbf{T}$ is onto, each $\mathbf{T}_{k}$ has size $\beta$ or $\beta+1$, and
    $|R_{\mathbf{T}}|
    =\sum_{k=1}^{K}\alpha_{k}
    =\alpha\sum_{k=1}^{K}a_{k}
    =\alpha\cdot\frac{n_{0}}{|C_{k^{*}}^{(0)}|}
    =\frac{r\,|C_{k^{*}}^{(0)}|}{n_{0}}\cdot\frac{n_{0}}{|C_{k^{*}}^{(0)}|}
    =r,$
    which completes the proof.
\end{proof}

\begin{proposition}\label{prop:matching4}
    Assume that $r > 0.$
    For a given $\mathbf{T},$ let $\bm{Z}$ with $Z_{j}^{(1)}=Z_{\mathbf{T}(j)}^{(0)}, j \in [n_{1}].$
    If $\mathbf{T}\in \mathcal{T}$ satisfied $|C_{k}^{(0)}|/n_0= |C_{k}^{(0)}\cap R_{\mathbf{T}}|/r$ for all $k \in [K],$
    we have that $\bm{Z} \in \mathcal{Z}^{\textup{Fair}}_0.$
\end{proposition}

\begin{proof}[Proof of \cref{prop:matching4}]
    Fix a $k\in[K]$.
    By definition of $\bm{Z}$ that
    $C_{k}^{(1)}=\{\,j:Z_{j}^{(1)}=k\}=\{\,j:Z_{\mathbf{T}(j)}^{(0)}=k\}$,
    we have that
    $$
    |C_{k}^{(1)}|
    =\sum_{i\in C_{k}^{(0)}}\bigl|\mathbf{T}^{-1}(i)\bigr|
    =\sum_{i\in C_{k}^{(0)}}\bigl(\beta+\mathbb{I} (i\in R_{\mathbf{T}}) \bigr)
    =\beta\,|C_{k}^{(0)}|+|\,C_{k}^{(0)}\cap R_{\mathbf{T}}\,|.
    $$
    By the sufficient condition $|\,C_{k}^{(0)}\cap R_{\mathbf{T}}\,|=(r/n_{0})\,|C_{k}^{(0)}|$ and since 
    $n_{1}=\beta\,n_{0}+r$, it follows that
    $$
    |C_{k}^{(1)}|
    =\beta\,|C_{k}^{(0)}|+\frac{r}{n_{0}}\,|C_{k}^{(0)}|
    =\frac{\beta\,n_{0}+r}{n_{0}}\,|C_{k}^{(0)}|
    =\frac{n_{1}}{n_{0}}\,|C_{k}^{(0)}|,
    $$
    which concludes the proof.
\end{proof}

\begin{proposition}\label{prop:relaxed_general}
    Assume that $r = 0.$
    Let $\mathbf{T}$ and $\mathbf{T}_{0}$ be the maps defined in \cref{sec:relax_G} for a given $m \le n_{1}.$
    Then, for any $\bm{Z}$ satisfying
    $ Z_{j}^{(1)}=Z_{\mathbf{T}(j)}^{(0)}, j \in [n_{1}] \setminus E $
    and
    $ Z_{j}^{(1)} = Z_{\mathbf{T}_{0}(j)}^{(0)}, j \in E, $
    we have that $\bm{Z} \in \mathcal{Z}_{m/n_{1}}^{\textup{Fair}}.$
\end{proposition}

\begin{proof}[Proof of \cref{prop:relaxed_general}]
    We first note the following two facts:
    (i) $\sum_{k=1}^{K} \vert \frac{1}{n_{0}}\sum_{i \in [n_{0}]}\mathbb{I}(Z_{i}^{(0)}=k) -\frac{1}{n_{1}}\sum_{j \in [n_{1}]}\mathbb{I}(Z_{\mathbf{T}(j)}^{(0)}=k) \rvert = 0$ for any given $\mathbf{T}$ due to the assumption $r = 0.$
    (ii) For any $\mathbf{T}$ and $\mathbf{T}_{0},$
    there exist nonnegative integers $m_{1},\dots,m_{K}$ with $\sum_{k}m_{k}=2m$ such that for each $k$,
    $\lvert\sum_{j \in E}\mathbb{I}(Z_{\mathbf{T}(j)}^{(0)}=k)-\sum_{j \in E}\mathbb{I}(Z_{\mathbf{T}_{0}(j)}^{(0)}=k)\rvert=m_{k}$.
    Therefore,
    \begin{equation*}
        \begin{split}
            \Delta(\bm{Z}) & = \frac{1}{2} \sum_{k=1}^{K}\biggl|\frac{1}{n_{0}}\sum_{i \in [n_{0}]}\mathbb{I}(Z_{i}^{(0)}=k) -\frac{1}{n_{1}}\sum_{j\in [n_{1}]}\mathbb{I}(Z_{j}^{(1)}=k)\biggr|
            \\
            & = \frac{1}{2} \sum_{k=1}^{K}\biggl|\frac{1}{n_{0}}\sum_{i \in [n_{0}]}\mathbb{I}(Z_{i}^{(0)}=k) -\frac{1}{n_{1}}\sum_{j \in [n_{1}] \setminus E}\mathbb{I}(Z_{j}^{(1)}=k) -\frac{1}{n_{1}}\sum_{j \in E}\mathbb{I}(Z_{j}^{(1)}=k)\biggr|
            \\
            & = \frac{1}{2} \sum_{k=1}^{K}\biggl|\frac{1}{n_{0}}\sum_{i \in [n_{0}]}\mathbb{I}(Z_{i}^{(0)}=k) -\frac{1}{n_{1}}\sum_{j \in [n_{1}] \setminus E}\mathbb{I}(Z_{\mathbf{T}(j)}^{(0)}=k) -\frac{1}{n_{1}}\sum_{j \in E}\mathbb{I}(Z_{\mathbf{T}_{0}(j)}^{(0)}=k)\biggr|
            \\
            & \le \frac{1}{2} \sum_{k=1}^{K}\biggl|\frac{1}{n_{0}}\sum_{i \in [n_{0}]}\mathbb{I}(Z_{i}^{(0)}=k) -\frac{1}{n_{1}}\sum_{j \in [n_{1}]}\mathbb{I}(Z_{\mathbf{T}(j)}^{(0)}=k)\biggr| 
            \\
            & + \frac{1}{2} \sum_{k=1}^{K} \biggl| \frac{1}{n_{1}} \sum_{j \in E} \mathbb{I}(Z_{\mathbf{T}(j)}^{(0)}=k) - \frac{1}{n_{1}} \sum_{j \in E} \mathbb{I}(Z_{\mathbf{T}_{0}(j)}^{(0)}=k) \biggr|
            \\
            & = 0 + \frac{1}{2} \frac{1}{n_{1}}\sum_{k=1}^{K}m_{k}=\frac{m}{n_{1}}.
        \end{split}
    \end{equation*}
    Thus $\Delta(\bm{Z}) \le m/n_{1}$, so $\bm{Z}\in\mathcal{Z}_{m/n_{1}}^{\textup{Fair}}$.
\end{proof}

\clearpage
\section{Details of FBC}\label{sec:fbc}

\subsection{Calculation of the acceptance probability in STEP 1 \texorpdfstring{(Sampling $\mathbf{T}, \mathbf{T}_{0}, E$) of \cref{sec:inference}}{in Section 5}}\label{sec:accept_prob-appen}

In this section, we provide a detailed explanation about the calculation of $\alpha'(\mathbf{T}',\mathbf{T}_{0}',E)$ in \cref{eq:acceptance_ratio}.
First, we have that
\begin{equation*}
    \begin{aligned}
        \alpha'(\mathbf{T}',\mathbf{T}_{0}',E')&=\frac{p(\mathbf{T}',\mathbf{T}_{0}',E' \vert \Phi,\mathcal{D}) q((\mathbf{T}',\mathbf{T}_{0}',E') \to (\mathbf{T},\mathbf{T}_{0},E))}{p(\mathbf{T},\mathbf{T}_{0},E \vert \Phi,\mathcal{D}) q((\mathbf{T},\mathbf{T}_{0},E) \to (\mathbf{T}',\mathbf{T}_{0}',E'))}
        \\
        & =
        \frac{p(\mathbf{T}',\mathbf{T}_{0}',E')\mathcal{L}(\mathcal{D};\Phi,\mathbf{T}',\mathbf{T}_{0}', E')
        q( (\mathbf{T}',\mathbf{T}_{0}', E') \to (\mathbf{T},\mathbf{T}_{0},E) ) }{p(\mathbf{T},\mathbf{T}_{0},E)\mathcal{L}(\mathcal{D};\Phi,\mathbf{T},\mathbf{T}_{0}, E) 
        q( (\mathbf{T},\mathbf{T}_{0},E) ) \to (\mathbf{T}',\mathbf{T}_{0}', E') )}
        \\
        & =\frac{e(\mathbf{T}')\mathcal{L}(\mathcal{D};\Phi,\mathbf{T}', \mathbf{T}_{0}', E')}{e(\mathbf{T})\mathcal{L}(\mathcal{D};\Phi,\mathbf{T},\mathbf{T}_{0}, E)},
    \end{aligned}
\end{equation*}
where the last equality holds since
\begin{equation*}
    \begin{split}
        q((\mathbf{T}', \mathbf{T}_{0}', E') \to (\mathbf{T},\mathbf{T}_{0},E)) & = q((\mathbf{T},\mathbf{T}_{0},E) \to (\mathbf{T}',\mathbf{T}_{0}',E'))
    \end{split}
\end{equation*}
and 
$
p (\mathbf{T}',\mathbf{T}_{0}',E') / p(\mathbf{T},\mathbf{T}_{0},E) = \mathbf{e}(\mathbf{T}') \big/ \mathbf{e}(\mathbf{T}),
$
because the priors of $\mathbf{T}_0$ and $E$ are uniform.

For the likelihood ratio, if we use a conjugate prior which enables the calculation of marginal likelihood,
we have
\begin{equation*}
    \frac{\mathcal{L}(\mathcal{D};\Phi,\mathbf{T}',\mathbf{T}_{0}',E')}{\mathcal{L}(\mathcal{D};\Phi,\mathbf{T},\mathbf{T}_{0},E)}
    = \frac{\prod_{c\in\mathcal{C}}m(X^{c}|\mathbf{T}',\mathbf{T}_{0}',E')}{\prod_{c\in\mathcal{C}}m(X^{c}|\mathbf{T},\mathbf{T}_{0},E)},
\end{equation*}
where
\begin{equation}\label{eq:marginal_likelihood}
    m(X^{c} \vert \mathbf{T}, \mathbf{T}_{0}, E)=\int_\Theta \left[ \prod_{i\in c} f(X_i^{(0)}|\phi_c) \prod_{j \in J(c; \mathbf{T}, \mathbf{T}_{0}, E)} f(X_j^{(1)}|\phi_c) \right] H(d\theta), 
\end{equation}
$J(c; \mathbf{T}, \mathbf{T}_{0}, E) := \{ j \in E: \mathbf{T}_{0}(j) \in c \} \cup \{ j \in [n_{1}] \setminus E: \mathbf{T}(j) \in c \}$ for $c\in\mathcal{C}$ and
$X^{c} = \{X_i^{(0)}:i\in c\}\cup\{X_j^{(1)}: j \in J(c; \mathbf{T}, \mathbf{T}_{0}, E) \}$ for $c \in \mathcal{C}.$
For a non-conjugate $H,$ we have
\begin{equation*}
    \frac{\mathcal{L}(\mathcal{D};\Phi,\mathbf{T}',\mathbf{T}_{0}',E')}{\mathcal{L}(\mathcal{D};\Phi,\mathbf{T},\mathbf{T}_{0},E)}
    = \frac{
    \prod_{c\in\mathcal{C}} \left[ \prod_{i\in c} f(X_i^{(0)}|\phi_c)
    \prod_{j \in J(c; \mathbf{T}', \mathbf{T}_{0}', E')} f(X_j^{(1)}|\phi_c) \right]
    }{
    \prod_{c\in\mathcal{C}} \left[ \prod_{i\in c} f(X_i^{(0)}|\phi_c)
    \prod_{j \in J(c; \mathbf{T}, \mathbf{T}_{0}, E)} f(X_j^{(1)}|\phi_c) \right]
    }.
\end{equation*}
Note that these calculations are derived from the equivalent representation in \cref{sec:equiv_rep}.

\clearpage
\subsection{Sampling algorithm for STEP 2 \texorpdfstring{(Sampling $\Phi$) of \cref{sec:inference}}{of Section 5}}\label{sec:details_STEP2s}

Here, we consider the following two cases.
When $m(X^{c} \vert \mathbf{T}, \mathbf{T}_{0}, E)$ can be easily computed (e.g., $H$ is a conjugate prior), we sample $\mathcal{C}\sim p(\mathcal{C}\vert\mathcal{D}, \mathbf{T}, \mathbf{T}_{0}, E).$
Otherwise, when $m(X^{c} \vert \mathbf{T}, \mathbf{T}_{0}, E)$ is intractable, we sample $(\mathcal{C},\bm{\phi})\sim p(\mathcal{C},\bm{\phi}\vert\mathcal{D}, \mathbf{T}, \mathbf{T}_{0}, E).$
If the marginal likelihood is computable, a direct adaptation of Algorithm 3 from \cite{neal2000markov,maceachern1998estimating} is applicable.
Otherwise, when the marginal likelihood is not computable, Algorithm 8 from \cite{neal2000markov} can be applied.
Wherever the meaning is clear, we abbreviate $J(c ; \mathbf{T}, \mathbf{T}_{0}, E)$ by $J(c)$ in this section.

\paragraph{Conjugate prior}
The modification of Algorithm 3 for FBC is described as below.
\begin{enumerate}
    \item Initialize $\mathcal{C}=\{[n_{0}]\}$ (i.e., a single cluster)
    \item Repeat the following steps $N$ times, to obtain $N$ samples.
    For $i=1,\ldots,n_{0}$: Remove element $i\in[n_{0}]$ and its matched elements in $J(\{ i\}; \mathbf{T}, \mathbf{T}_{0}, E) := \{ j \in E: \mathbf{T}_{0}(j) = i \} \cup \{ j \in [n_{1}] \setminus E: \mathbf{T}(j) = i \}$ from $\mathcal{C}$. 
    Then, place them 
    \begin{itemize}[topsep=0pt]
        \item to $c'\in\mathcal{C}\setminus i$ with probability 
        \begin{equation*}
            \propto(|c'|+\gamma)\frac{m(X^{c'}\cup X^{\{ i \}}|\mathbf{T}, \mathbf{T}_{0}, E)}{m(X^{c'}|\mathbf{T}, \mathbf{T}_{0}, E)}
        \end{equation*}
        where $X^{\{ i \}}:=\{X_i^{(0)}\}\cup\{ X_j^{(1)}:j\in J(\{ i\}) \}$.
        \item to a new cluster with probability 
        \begin{equation*}
            \propto\gamma\frac{V_{n_0}(t+1)}{V_{n_0}(t)}m(X^{\{ i \}}|\mathbf{T}, \mathbf{T}_{0}, E)
        \end{equation*}
        where $t$ is a number of clusters when $X^{\{ i \}}$ are removed. 
    \end{itemize}
\end{enumerate}

\begin{proposition}\label{prop:alg_validity}
    The above modification of Algorithm 3 of \cite{neal2000markov} is a valid Gibbs sampler.
\end{proposition}
\begin{proof}[Proof of \cref{prop:alg_validity}]
    
    The posterior density can be formulated as follows:
    \begin{align*}
        p\left(\mathcal{C}\big\vert \mathcal{D}, \mathbf{T},\mathbf{T}_0,E\right)
        &\propto p(\mathcal{C})\cdot p\left(X_{1:n_0}^{(0)},X_{1:n_1}^{(1)} \big\vert \mathcal{C},\mathbf{T},\mathbf{T}_0,E \right)\\
        &=V_{n_0}(t) \prod_{c\in\mathcal{C}}\gamma^{(|c|)}\cdot \prod_{c\in\mathcal{C}}m(X^c|\mathbf{T},\mathbf{T}_0,E) 
    \end{align*}
        
    To justify the proposed modification of Algorithm 3, we only need to calculate the probabilities of placing $X_i^{(0)}$ and its matched elements $\{X_j^{(1)};j\in J(c(i))\}$: 
    (i) to an existing partition $c'$, 
    or
    (ii) to a new cluster. 
    Let $\mathcal{C}_{-i}$ be the collection of clusters where $X_i^{(0)}$ and its matched elements $\{X_j^{(1)}:j\in J(\{i\})\}$ are removed from $\mathcal{C}.$ 
    The calculation can be done as follows: 
    \begin{enumerate}
        \item[(i)] to existing $c'$: 
        
        The term $\gamma^{(|c'|)}$ in the prior term $p(\mathcal{C})\propto V_{n_0}(t)\prod_{c\in\mathcal{C}}\gamma^{(|c|)}$ changes to $\gamma^{(|c'|+1)}$ for this particular $c'$.
        The marginal likelihood $m(X^{c'})$ changes to $m(X^{c'}\cup X^{\{i\}})$ for this particular $c'$.
        Hence, we have the following conditional probability: 
        \begin{align*}
            p(i\to c'| \mathcal{C}_{-i},\mathcal{D},\mathbf{T},\mathbf{T}_0,E)
            & \propto\frac{\gamma^{(|c'|+1)}}{\gamma^{(|c'|)}}\frac{m(X^{c'}\cup X^{\{ i \}}|\mathbf{T}, \mathbf{T}_{0}, E)}{m(X^{c'} | \mathbf{T},\mathbf{T}_0,E)}
            \\
            &= (|c'|+\gamma)\frac{m(X^{c'}\cup X^{\{ i \}} \}|\mathbf{T},\mathbf{T}_0,E)}{m(X^{c'}|\mathbf{T},\mathbf{T}_0,E)}
        \end{align*}
        
        \item[(ii)] to a new cluster: 
        
        The prior term $V_{n_0}(t)\prod_{c\in\mathcal{C}}\gamma^{(|c|)}$ changes into 
        $V_{n_0}(t+1)\left[\prod_{c\in\mathcal{C}}\gamma^{|c|}\right]\gamma$. 
        Hence, we have the following conditional probability: 
        \begin{align*}
            &p(i\to \textup{new}| \mathcal{C}_{-i},\mathcal{D},\mathbf{T},\mathbf{T}_0,E)\\
            & \propto \gamma\frac{V_{n_0}(t+1)}{V_{n_0}(t)}m(X^{\{ i \}}\vert \mathbf{T},\mathbf{T}_0,E)
        \end{align*}
    \end{enumerate}
\end{proof}

\paragraph{Non-conjugate prior}

When using a non-conjugate prior, we can use Algorithm 8 instead of Algorithm 3.
The outline of implementation of Algorithm 8 for FBC can be formulated similar to those of Algorithm 3, as below. 

\begin{enumerate}
    \item Initialize $\mathcal{C}=\{[n_{0}]\}$ (i.e., a single cluster) with $\phi_{[n_0]}\sim H$. 
    \item Repeat the following steps $N$ times, to obtain $N$ samples.
    For $i=1,\ldots,n_{0}$: Remove element $i\in[n_{0}]$ and its matched elements in $J(\{ i \})$ from $\mathcal{C}$. 
    Then, generate $m$ independent auxiliary variables $\phi^{(1)},\ldots,\phi^{(m)}\sim H$. 
    Compute the assignment weights as:
    \begin{align*}
    w_{c'} &=(|c'|+\gamma)
             \prod_{x\in X^{\{ i\}}}f(x\vert\phi_{c'}),
             & c'\in\mathcal C\setminus{i},\\[2pt]
    w_{\textup{aux},h'}
           &=\frac{\gamma}{m}
             \frac{V_{n_0}(t+1)}{V_{n_0}(t)}
             \prod_{x\in X^{\{ i\}}}f(x\vert\phi^{(h')}),
             & h'=1,\dots,m ,
    \end{align*}
    where $X^{\{ i \}}:=\{X_i^{(0)}\}\cup\{ X_j^{(1)}:j\in J(\{ i \}) \}$ and $t$ is a number of clusters when $X^{\{ i \}}$ are removed. 
    Then, place them 
    \begin{itemize}[topsep=0pt]
        \item to $c'\in\mathcal{C}\setminus i$ with probability $\propto w_{c'}$, or
        \item to a new randomly chosen cluster $h$ among $m$ auxiliary components, with probability $\propto w_{\textup{aux},h}$.
    \end{itemize}
    Then, discard all auxiliary variables which are not chosen. 
\end{enumerate}

\begin{proposition}\label{prop:alg_validity_nonconjugate}
    The above modification of Algorithm 8 of \cite{neal2000markov} is a valid Gibbs sampler.
\end{proposition}
\begin{proof}[Proof of \cref{prop:alg_validity_nonconjugate}]    
    To justify the proposed modification of Algorithm 8, we only need to calculate the probabilities of replacing $X^{\{ i \}}$: 
    (i) to an existing partition $c'$, 
    or
    (ii) to a new cluster. 
    
    The calculation can be done as follows: 
    \begin{enumerate}
        \item[(i)] to existing $c'$: 
        
        The term $\gamma^{(|c'|)}$ in the prior term $p(\mathcal{C})\propto V_{n_0}(t)\prod_{c\in\mathcal{C}}\gamma^{(|c|)}$ changes to $\gamma^{(|c'|+1)}$ for this particular $c'$.
        The likelihood is multiplied by $\prod_{x\in X^{\{ i \}}}f(x\vert\phi_{c'})$ for this particular $c'$.
        Hence, we have the following conditional probability: 
        \begin{align*}
            p(i\to c'| \mathcal{C}_{-i},\mathcal{D},\mathbf{T},\mathbf{T}_0,E)
            & \propto\frac{\gamma^{(|c'|+1)}}{\gamma^{(|c'|)}}\prod_{x\in X^{\{ i \}}}f(x\vert\phi_{c'})
            \\
            &= (|c'|+\gamma)\prod_{x\in X^{\{ i \}}}f(x\vert\phi_{c'}).
        \end{align*}
        
        \item[(ii)] to a new cluster: 
        A new cluster $h$ is chosen uniformly among the auxiliary $m$ components. 
        Then, we have the following conditional probability by Monte-Carlo approximation:
        \begin{align*}
            p(i\to \textup{new}| \mathcal{C}_{-i},\mathcal{D},\mathbf{T},\mathbf{T}_0,E)
            &\propto \gamma\frac{V_{n_0}(t+1)}{V_{n_0}(t)}m(X^{\{ i \}}\vert \mathbf{T},\mathbf{T}_0,E)\\
            &= \gamma\frac{V_{n_0}(t+1)}{V_{n_0}(t)}\int_\Theta \left[   \prod_{x \in X^{\{ i \}}} f(x|\phi) \right] H(d\phi) \\
            &\approx \gamma\frac{V_{n_0}(t+1)}{V_{n_0}(t)}\frac{1}{m}\sum_{h'=1}^{m}\prod_{x\in X^{\{ i \}}}f(x|\phi^{(h')})\\
            &=\sum_{h'=1}^{m}w_{\textup{aux},h'}.
        \end{align*}
        From the fact that $h$ is uniformly chosen among $m$-auxiliary components, we have the conditional probability to a new cluster $h$ as follows:
        \begin{align*}
            p(i\to h|\mathcal{C}_{-i},\mathcal{D},\mathbf{T},\mathbf{T}_0,E)
            &\propto p(i\to h|i\to\textup{new},\cdot)p(i\to\textup{new}|\mathcal{C}_{-i},\mathcal{D},\mathbf{T},\mathbf{T}_0,E)\\
            &\propto\frac{w_{\textup{aux},h}}{\sum_{h'=1}^{m}w_{\textup{aux},h'}}\cdot\sum_{h'=1}^{m}w_{\textup{aux},h'}\\
            &=w_{\textup{aux},h}.
        \end{align*}
    \end{enumerate}
\end{proof}

\subsection{Pseudo-code of the overall FBC algorithm}\label{sec:details_overall}

\cref{alg:FBC} below provides the pseudo-code of our proposed MCMC algorithm, which is a combination of the two steps in \cref{sec:alg-step1,sec:alg-step2}. 
\begin{algorithm}[h]
    \caption{FBC algorithm}
    \begin{algorithmic}[1]
        \STATE \textbf{Inputs:}
        Data ($\mathcal{D} = \mathcal{D}^{(0)} \cup \mathcal{D}^{(1)}$),
        Maximum number of iterations for inference ($\textup{max}_{\textup{iter}}$).

        \STATE Initialize $\Phi^{(0)},\mathbf{T}^{(0)},\mathbf{T}_{0}^{(0)},E^{(0)}$
        \FOR{$t = 1$ to $\textup{max}_{\textup{iter}}$}
            \STATE (STEP 1) Propose $(\mathbf{T}',\mathbf{T}_{0}',E')$ and compute the acceptance probability $\alpha(\mathbf{T}',\mathbf{T}_{0}',E')$
            \STATE Sample $u \sim \textup{Uniform}(0,1)$
            \IF{$u<\alpha(\mathbf{T}',\mathbf{T}_{0}',E')$}
                \STATE Accept) $\mathbf{T}^{(t)}=\mathbf{T}',\mathbf{T}_{0}^{(t)'}=\mathbf{T}_{0},E^{(t)}=E'$
            \ELSE
                \STATE Reject) $\mathbf{T}^{(t)}=\mathbf{T}^{(t-1)},\mathbf{T}_{0}^{(t)}=\mathbf{T}_{0}^{(t-1)},E^{(t)}=E^{(t-1)}$
            \ENDIF
            \STATE (STEP 2) Sample $\Phi^{(t)}$ from posterior $p(\Phi|\mathbf{T}^{(t)},\mathbf{T}_{0}^{(t)},E^{(t)},\mathcal{D})$
        \ENDFOR
        
        \STATE \textbf{Return:}
        Posterior samples $\left\{\left(\Phi^{(t)},\mathbf{T}^{(t)},\mathbf{T}_{0}^{(t)},E^{(t)}\right)\right\}_{t=1}^{\textup{max}_\textup{iter}}$.
    \end{algorithmic}
    \label{alg:FBC}
\end{algorithm}


\clearpage
\section{Extension of FBC for a multinary sensitive attribute}\label{sec:multiple_s}
This section explains that FBC can be modified for the case of a multinary sensitive attribute (i.e., the number of groups $\ge 3$). For simplicity, we consider three sensitive groups.
Extension to more than three groups can be done similarly.

Let $\mathcal{D}^{(2)} = \{ X_i^{(2)} \}_{i=1}^{n_{2}},$ along with existing $\mathcal{D}^{(0)},\mathcal{D}^{(1)}.$ 
Assume that $n_0 \le \min \{n_1,n_2\}.$
Similar to the binary sensitive case, we consider $\mathbf{T}_{1}: [n_{1}] \to [n_{0}]$ and $\mathbf{T}_{2}: [n_{2}] \to [n_{0}]$ as random matching maps from $\mathcal{D}^{(1)}$ to $\mathcal{D}^{(0)}$ and from $\mathcal{D}^{(2)}$ to $\mathcal{D}^{(0)},$ respectively. 
We also consider arbitrary functions $\mathbf{T}_{01}:[n_1]\to[n_0]$ and $\mathbf{T}_{02}:[n_2]\to[n_0]$ and arbitrary subsets $E_1\in[n_1],E_2\in[n_2]$ of sizes $m_1$ and $m_2$.
Let $\mathcal{C}$ be a partition of $[n_{0}]$ induced by $ \bm{Z} $ such that $Z_i^{(0)} \vert \bm{\pi} \overset{\textup{i.i.d.}}{\sim} \textup{Categorical} (\cdot \vert \bm{\pi}), \forall i \in [n_{0}].$
Then, similar to \cref{sec:equiv_rep}, we consider the generative model 
\begin{align*}
    & \phi_c\overset{\textup{i.i.d.}}{\sim}H, c\in\mathcal{C}
    \\
    & X_i^{(0)} \overset{\textup{ind}}{\sim} f(\cdot \vert \phi_c), i \in c
    \\
    & X_j^{(1)} \overset{\textup{ind}}{\sim}
    \begin{cases}
        f(\cdot\vert\phi_c) & \forall j\in[n_1]\setminus E_1 \textup{ s.t. } \mathbf{T}_{1}(j)\in c \\
        f(\cdot\vert\phi_c) & \forall j\in E_1 \textup{ s.t. } \mathbf{T}_{01}(j)\in c
    \end{cases}
    \\
    & X_j^{(2)} \overset{\textup{ind}}{\sim}
    \begin{cases}
        f(\cdot\vert\phi_c) & \forall j\in[n_2]\setminus E_2 \textup{ s.t. } \mathbf{T}_{2}(j)\in c \\
        f(\cdot\vert\phi_c) & \forall j\in E_2 \textup{ s.t. } \mathbf{T}_{02}(j)\in c
    \end{cases}
\end{align*}
The inference algorithm can be also modified accordingly.
Furthermore, $\Delta(\mathbf{Z})$ is also generalized as follows.
$G$
\paragraph{Fairness measure for a multinary sensitive attribute}
Let $B$ be the number of sensitive groups.
For $b \in \{0, 1, \dots, B-1\},$ let $n_b$ be the number of samples in group $b$, and denote $\{ Z_{i}^{(b)} \}_{i=1}^{n_{b}}$ as the cluster assignments of group $b$.
Then $\Delta(\bm{Z})$ is generalized as:
$$
\Delta(\bm{Z}) := \frac{1}{2(B-1)} \sum_{k=1}^{K} \sum_{b=1}^{B-1}
\left|
\frac{1}{n_{0}} \sum_{i=1}^{n_{0}} \mathbb{I}(Z^{(0)}_i = k)
-
\frac{1}{n_{b}} \sum_{j=1}^{n_{b}} \mathbb{I}(Z^{(b)}_j = k)
\right|\in[0,1]
$$
The \texttt{Bal} is also generalized as:
$\texttt{Bal} := \min_{k \in [K]} \texttt{Bal}_{k}$
where
$$
\texttt{Bal}_k := \min_{b_1 \neq b_2} \left\{
\frac{|C_k^{(b_1)}|}{|C_k^{(b_2)}|}, \frac{|C_k^{(b_2)}|}{|C_k^{(b_1)}|}
\right\},
$$
where $ C_k^{(b)} $ denotes the set of samples in group $b$ assigned to the $k^{\textup{th}}$ cluster. 
Note that the above formulation of $\Delta(\bm{Z})$ with $B=2$ coincides with the definition of $\Delta(\bm{Z})$ in the binary sensitive case, which is defined in \cref{eq:def_delta}. 
\cref{prop:relaxed_general_3group} below further shows that FBC can control the fairness level $\Delta(\bm{Z})$ for a multinary sensitive attribute.

\begin{proposition}\label{prop:relaxed_general_3group}
    Denote $n_0$, $n_1$, and $n_2$ as the number of samples in three sensitive groups.  
    Let $\mathbf{T}_1 : [n_1] \to [n_0]$ and $\mathbf{T}_2 : [n_2] \to [n_0]$ be matching maps from groups $1$ and $2$ to group $0.$
    Let consider arbitrary functions $\mathbf{T}_{01}:[n_1]\to[n_0], \mathbf{T}_{02}:[n_2]\to[n_0]$ and arbitrary subsets $E_1\in[n_1],E_2\in[n_2]$ of sizes $m_1$ and $m_2$.
    Suppose that the assignment $\bm{Z}$ satisfies:
    \begin{itemize}
        \item $Z_j^{(1)} = Z_{\mathbf{T}_1(j)}^{(0)}$ for $j \in [n_1] \setminus E_1$ and $Z_j^{(1)} = Z_{\mathbf{T}_{01}(j)}^{(0)}$ for $j \in E_1$,
        \item $Z_j^{(2)} = Z_{\mathbf{T}_2(j)}^{(0)}$ for $j \in [n_2] \setminus E_2$ and $Z_j^{(2)} = Z_{\mathbf{T}_{02}(j)}^{(0)}$ for $j \in E_2$.
    \end{itemize}
    Then, we have
    $$
    \Delta(\bm{Z}) \le \frac{1}{2} \left( \frac{m_1}{n_1} + \frac{m_2}{n_2} \right).
    $$
\end{proposition}

\begin{proof}[Proof of \cref{prop:relaxed_general_3group}]
    We investigate all three pairs among the three groups: $(0,1)$, $(0,2)$, and $(1,2)$.
    
    {(i) Between groups 0 and 1:}  
    Since $Z_j^{(1)} = Z_{\mathbf{T}_1(j)}^{(0)}$ for $j \in [n_1] \setminus E_1,$
    we have the following inequality utilizing the proof of \cref{prop:relaxed_general}, 
    $$
    \frac{1}{2} \sum_{k=1}^K \left| \frac{1}{n_0} \sum_{i=1}^{n_0} \mathbb{I}(Z_i^{(0)} = k) 
    - \frac{1}{n_1} \sum_{j=1}^{n_1} \mathbb{I}(Z_j^{(1)} = k) \right| \le \frac{m_1}{n_1}.
    $$
    
    {(ii) Between groups 0 and 2:}
    Similarly, we have
    $$
    \frac{1}{2} \sum_{k=1}^K \left| \frac{1}{n_0} \sum_{i=1}^{n_0} \mathbb{I}(Z_i^{(0)} = k) 
    - \frac{1}{n_2} \sum_{j=1}^{n_2} \mathbb{I}(Z_j^{(2)} = k) \right| \le \frac{m_2}{n_2}.
    $$
    
    {Combining the two terms:}
    Taking the average, we get:
    $$
    \Delta(\bm{Z}) = \frac{2}{2(3 - 1)} \sum_{b=1}^{2} 
    \left( \frac{1}{2} \sum_{k=1}^K \left| \frac{1}{n_{0}} \sum_{i=1}^{n_{0}} \mathbb{I}(Z_i^{(0)} = k) 
    - \frac{1}{n_{b}} \sum_{j=1}^{n_{b}} \mathbb{I}(Z_j^{(b)} = k) \right| \right)
    $$
    
    $$
    \le \frac{1}{2} \left( \frac{m_1}{n_1} + \frac{m_2}{n_2} \right).
    $$
\end{proof}

Note that we use this extended approach for \textsc{Bank} dataset with three sensitive groups in our experiments.
See \cref{sec:exps-appen_results} for the results showing that FBC works well for a multinary sensitive attribute. 

\clearpage
\section{Experiments}\label{sec:exps-appen}

\subsection{Datasets}\label{sec:exps-appen_data}

\begin{enumerate}
    \item 
    Toy dataset:
    We build a 2D toy dataset from a 6-component Gaussian mixture model with unit covariance matrix $\mathbb{I}_{2}.$
    For $\mathcal{D}^{(0)},$ we draw 600 samples from each $\mathcal{N}([-5, -30], \mathbb{I}_{2})/3 +\mathcal{N}([-5, 0], \mathbb{I}_{2})/3+ \mathcal{N}([-5, 30], \mathbb{I}_{2})/3$ and
    Similarly for $\mathcal{D}^{(1)},$ we draw 600 samples from $\mathcal{N}([-5, -29.5], \mathbb{I}_{2})/3+ \mathcal{N}([-5, 0.5], \mathbb{I}_{2})/3+ \mathcal{N}([-5, 30.5], \mathbb{I}_{2})/3.$
    As a result, the total number of samples is $1200,$ with $n_{0} = n_{1} = 600.$

    \item 
    \textsc{Diabetes}: 
    The diabetes dataset is a collection of data spanning five years, consisting of various physical indicators (e.g., glucose concentration, blood pressure, BMI, etc., totaling 7 features) of Pima Indian women\footnote{Downloaded from \url{https://github.com/aasu14/Diabetes-Data-Set-UCI}}.
    The sample size is 768 (sensitive group sizes: 396 and 372).
    It originates from the National Institute of Diabetes and Digestive and Kidney Diseases \citep{pima_diabetes}.
    For the sensitive attribute, we use the binarized age attribute at the median value.

    \item 
    \textsc{Adult}:
    The adult income dataset is a collection of data consisting of several demographic features including employment features.
    It is extracted from 1994 U.S. Census database \citep{misc_adult_2}.
    We subsample 1,000 data points (sensitive group sizes: 694 and 306) from the original dataset.
    We use 5 continuous features ({age}, {fnlwgt}, {education-num}, {capital-gain}, {hours-per-week}).
    For the sensitive attribute, we use the gender (male/female) attribute.

    \item 
    \textsc{Bank}:
    The bank marketing dataset is a collection of data from a Portuguese bank’s direct marketing campaigns, each corresponding to an individual client contacted \citep{MORO201422}.
    We use 6 continuous features (age, call duration, 3-month Euribor rate, number of employees, consumer price index, and number of contacts during the campaign).
    For two sensitive groups, we treat marital status as the sensitive attribute: categorized into two groups (single/married) and exclude all `unknown' entries.
    We subsample 1,000 data points (sensitive group sizes: 606 and 394) from the original dataset.
    
    To consider three sensitive groups, following \citep{ziko2021variational}, we categorize the marital status into three groups (single/married/divorced) and exclude all `unknown' entries.
    We subsample 1,000 data points (sensitive group sizes: 586, 305, and 109) from the original dataset.
\end{enumerate}

\subsection{Implementation details}\label{sec:exps-appen_impl}

\paragraph{Algorithms}

\begin{itemize}
    \item Baseline methods:
    For MFM, we employ the Julia code of \cite{miller2018mixture} without modification, available on the authors' GitHub\footnote{\url{https://github.com/jwmi/BayesianMixtures.jl}}.
    Similarly, for SFC and VFC, we use the publicly released source codes provided by the authors\footnote{SFC: \url{https://github.com/talwagner/fair_clustering}} \footnote{VFC: \url{https://github.com/imtiazziko/Variational-Fair-Clustering}}.
    Particularly for VFC, due to overflow, we follow the authors’ implementation of using $L_{2}$-normalized data when running the algorithm and subsequently transform them to the original data by multiplying by the norm when calculating \texttt{Cost}, $\texttt{Bal},$ and $\Delta.$
    Fair MFM is a combination of SFC and MFM, where we find fairlets using SFC first and apply the MFM algorithm of \cite{miller2018mixture} to the space of the fairlets.
    
    \item FBC:
    We use a conjugate prior for $H.$
    For Gaussian mixture, we use the following prior specification.
    For $\bm{\theta},$ we set $ a = b =1$ in $\mathcal{N}(\mu_j,\lambda_j^{-1})$, $\lambda_j \sim \textup{Gamma}(a,b)$, $\mu_j \vert \lambda_j \sim \mathcal{N}(0, \lambda_j^{-1}).$
    For $K,$ we set $\kappa = 0.1$ in $K \sim \textup{Geometric}(\kappa)$ and set $\gamma=1$ in the Dirichlet distribution.
    The same priors are used for MFM and Fair MFM.
    The Julia language is used for running FBC. 

    For the choice of $D$ in the energy, we use the Euclidean distance for continuous features and the Hamming distance for categorical data \cite{hamming1950error,huang1998extensions,zhang2006clustering}.     
\end{itemize}

\paragraph{Hardwares}

\begin{itemize}
    \item All our experiments are done through Julia 1.11.2, Python 3.9.16 with Intel(R) Xeon(R) Silver 4310 CPU @ 2.10GHz and 128GB RAM. 
\end{itemize}

\clearpage
\subsection{Omitted experimental results}\label{sec:exps-appen_results}

\subsubsection{Fair clustering performance}\label{sec:exp-compare_appen}

\paragraph{Fairness level control of FBC}

Figure \ref{fig:control} shows the relationship between $m$ and the fairness level, showing that fairness level is well controlled by controlling $m$ unless $m$ is too large.

\begin{figure*}[h]
    \vskip -0.1in
    \centering
    \subfloat{\includegraphics[width=0.3\linewidth]{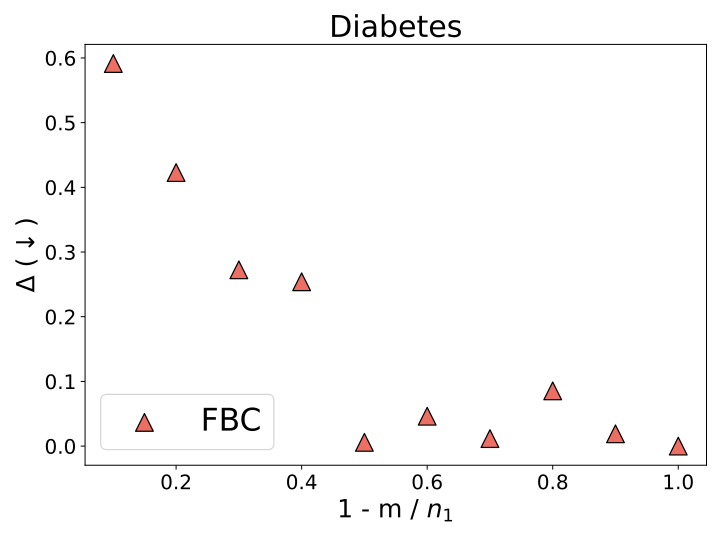}}
    \subfloat{\includegraphics[width=0.3\linewidth]{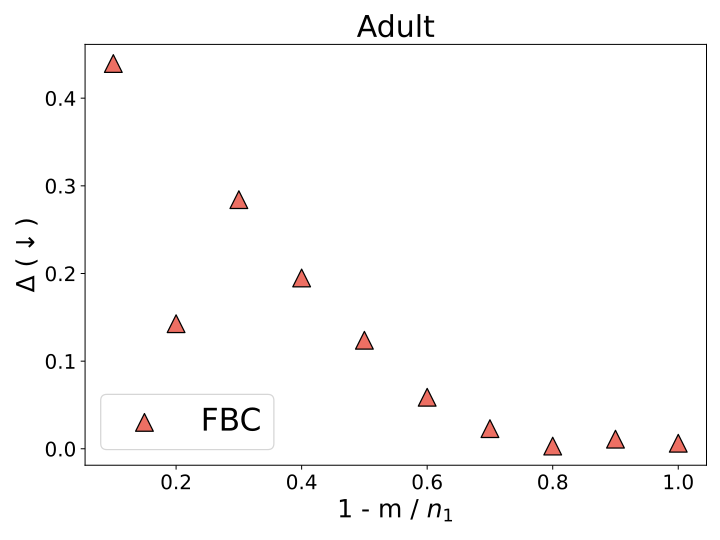}}
    \subfloat{\includegraphics[width=0.3\linewidth]{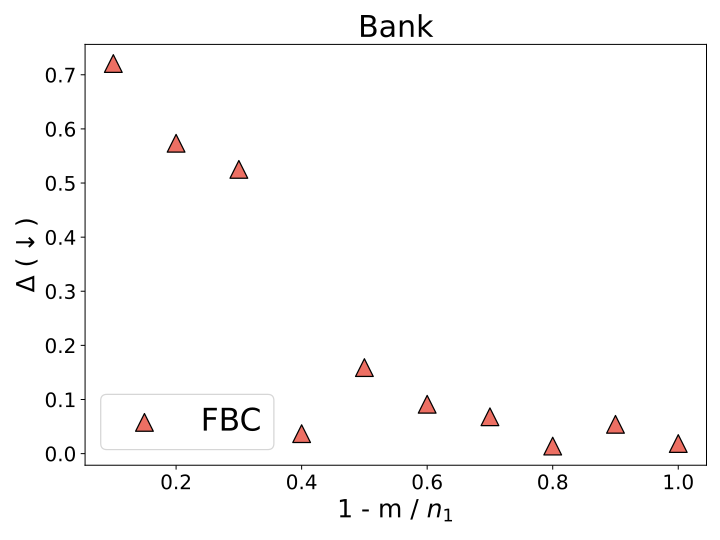}}
    \\
    \subfloat{\includegraphics[width=0.3\linewidth]{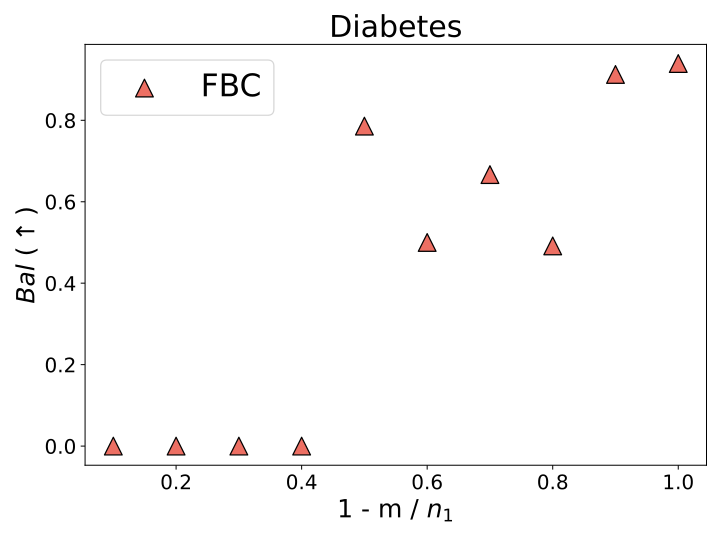}}
    \subfloat{\includegraphics[width=0.3\linewidth]{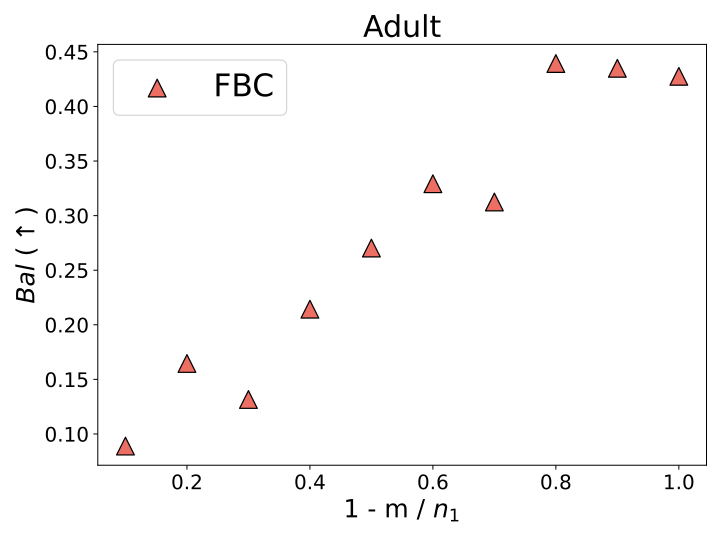}}
    \subfloat{\includegraphics[width=0.3\linewidth]{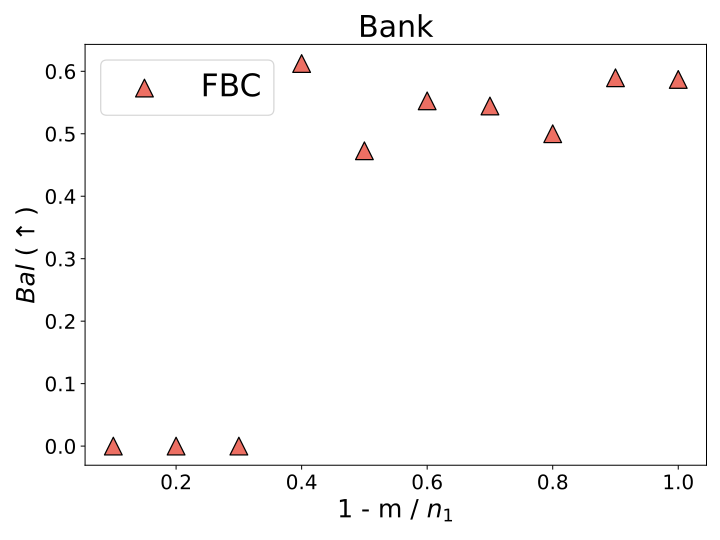}}
    \caption{
    (Top three) Trade-off between $m$ (the size of $E$) and the fairness level $\Delta.$
    Smaller $\Delta,$ fairer the clustering.
    (Bottom three) Trade-off between $m$ (the size of $E$) and the \texttt{Bal}.
    Larger \texttt{Bal}, fairer the clustering.
    }
    \vskip -0.1in
    \label{fig:control}
\end{figure*}

\paragraph{Analysis with a multinary sensitive attribute \texorpdfstring{(\textsc{Bank})}{(Bank)}}

We analyze \textsc{Bank} with three sensitive groups.
\cref{table:bank3-compare} presents the performance comparison between VFC and FBC, showing that FBC is competitive to VFC in terms of the utility–fairness trade-off.  
In particular, FBC achieves significantly better fairness levels (i.e., lower $\Delta$ and higher \texttt{Bal}), while its utility (\texttt{Cost}) remains comparable.
Moreover, \cref{fig:control-bank3} shows the relationship between $m$ and the fairness level, showing that a smaller $m$ results in a fairer clustering.

\begin{table}[h]
    \vskip -0.1in
    \footnotesize
    \caption{
    Comparison of \texttt{Cost}, $\Delta,$ and $\texttt{Bal}$ for VFC and FBC on \textsc{Bank} dataset with three sensitive groups.
    }
    \label{table:bank3-compare}
    \centering
    \vskip 0.1in
    \begin{tabular}{c||c|c|c|c}
        \toprule
        Dataset & \multicolumn{4}{c}{\textsc{Bank}}
        \\
        \midrule
        Method & $K$ & \texttt{Cost} ($\downarrow$) & $\Delta$ ($\downarrow$)& $\texttt{Bal}$ ($\uparrow$)
        \\
        \midrule
        VFC & 3 & 5.804 & 0.064 & 0.166
        \\
        FBC $\checkmark$ & 3 & 5.984 & 0.007 & 0.181 
        \\
        \bottomrule
    \end{tabular}
    \vskip -0.1in
\end{table}

\begin{figure*}[h]
    \vskip -0.1in
    \centering
    \subfloat{\includegraphics[width=0.3\linewidth]{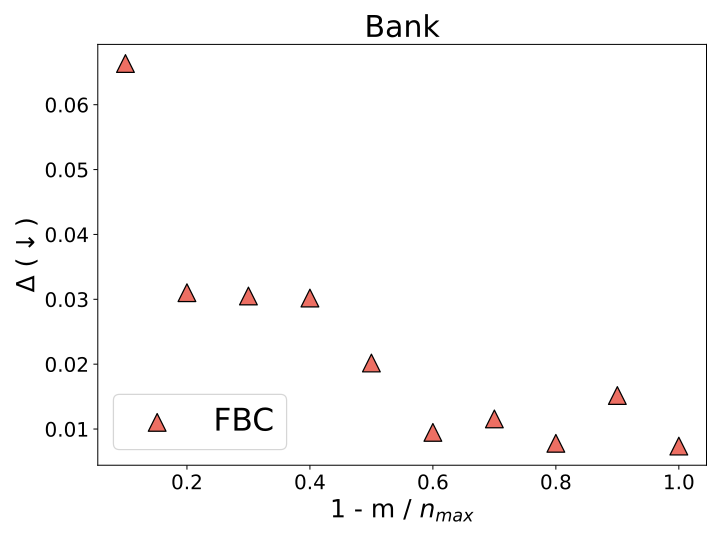}}
    \subfloat{\includegraphics[width=0.3\linewidth]{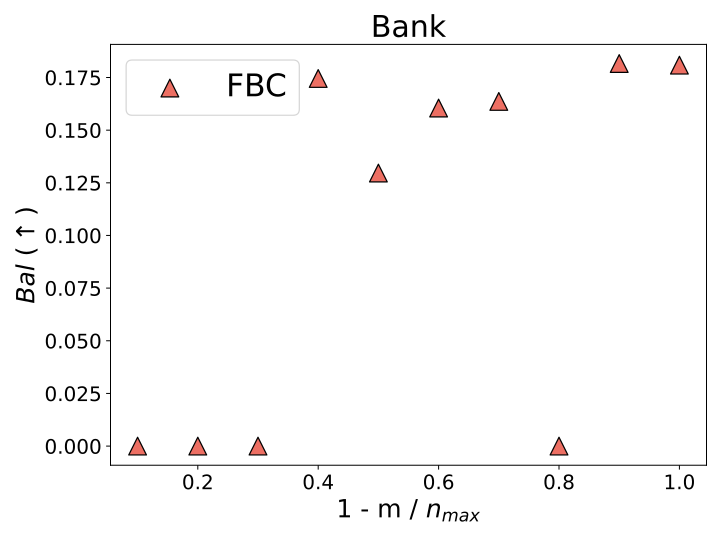}}
    \caption{
    Trade-off between the size of $E$ and the fairness level $\Delta.$
    Smaller $\Delta,$ fairer the clustering.
    The $x$-axis is $1 - m / n_{\textup{max}}$
    where $n_{\textup{max}} = \max (n_{0}, n_{1}, n_{2}) = \max (109, 305, 586) = 586.$
    }
    \vskip -0.1in
    \label{fig:control-bank3}
\end{figure*}

\clearpage
\subsubsection{Inference of the number of clusters \texorpdfstring{$K$}{K}}\label{sec:appen-exp-K_nll}

To see whether the inferred $K$ by FBC is reasonable, we calculate the negative log-likelihood (\texttt{NLL}) of the fair mixture model learned by FBC on test data as follows.

\begin{enumerate}
\item We prepare test dataset $\mathcal{D}_0^{\text{test}}$ and $\mathcal{D}_1^{\text{test}}$, each containing $n_{\text{test}}$ samples.
That is, $|\mathcal{D}_0^{\text{test}}| = |\mathcal{D}_1^{\text{test}}| = n_{\text{test}}$.

\item To construct fair assignments $\bm{Z}^{\textup{test}} := (Z_{1}^{(0), \text{test}}, \ldots, Z_{n_{\text{test}}}^{(0), \text{test}}, Z_{1}^{(1), \text{test}}, \ldots, Z_{n_{\text{test}}}^{(1), \text{test}})$ on $\mathcal{D}_0^{\text{test}} \cup \mathcal{D}_1^{\text{test}}$, we build a one-to-one matching map between $\mathcal{D}_0^{\text{test}}$ and $\mathcal{D}_1^{\text{test}}$, using the optimal transport map $\mathbf{T}_{\text{test}}$.

\item We match instances in $\mathcal{D}_0^{\text{test}}$ to those in  $\mathcal{D}_0$ by using an optimal transport $\mathbf{T}_{*}.$

\item We assign $Z_j^{(0), \text{test}}=Z_{\mathbf{T}_*(j)}^{(0)}$ and $Z_j^{(1), \text{test}} = Z_{\mathbf{T}_{\text{test}}(j)}^{(0), \text{test}}$ for $j \in [n_{\text{test}}].$

\item For each posterior sample obtained by FBC, we calculate the log-likelihood of test data conditional on $\bm{Z}^{\textup{test}}.$

\end{enumerate}

See \cref{fig:K_nll_visual} for a simple visualization of this construction when $n_{\textup{test}} = 4.$

\begin{figure}[h]
    \centering
    \includegraphics[width=0.7\linewidth]{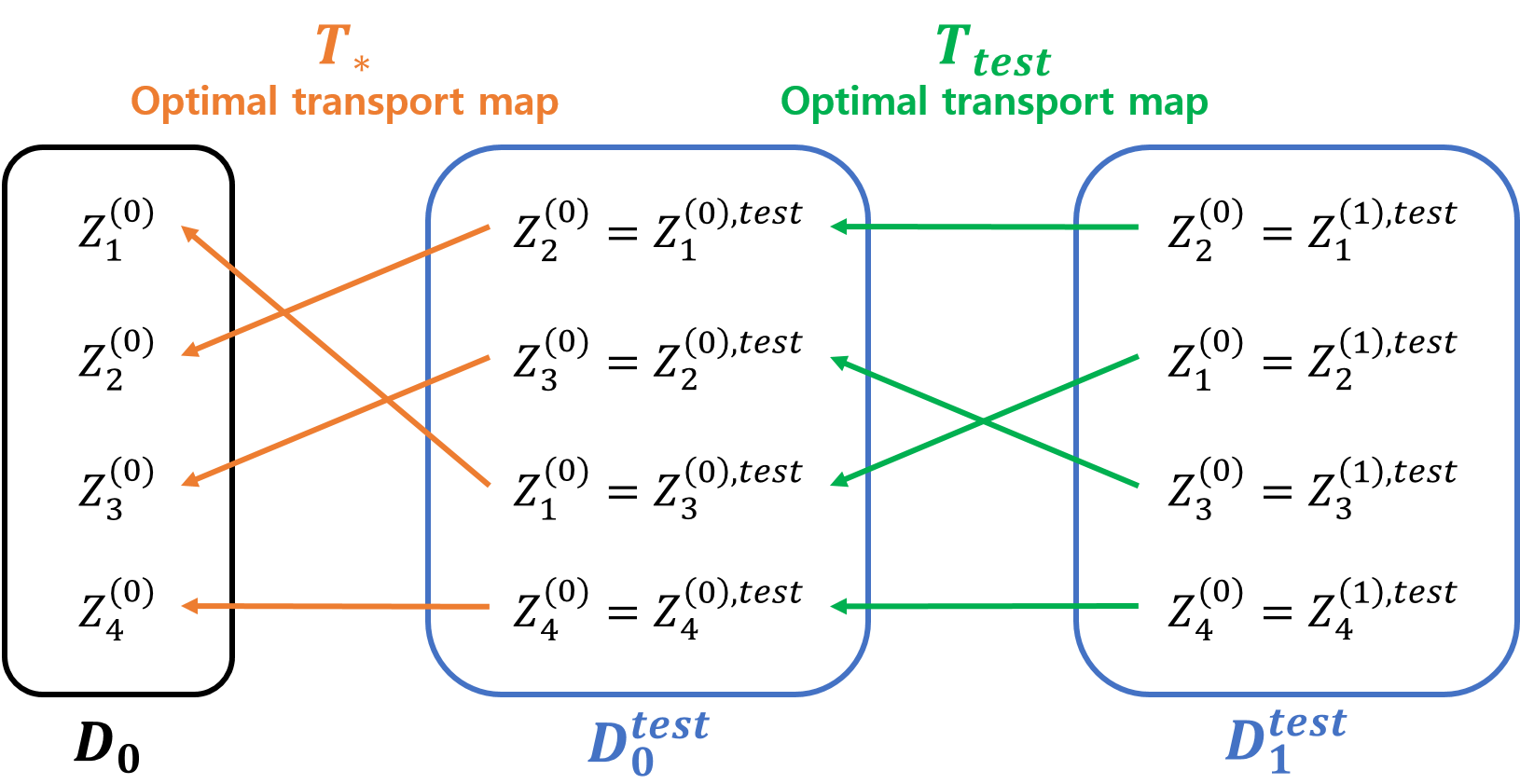}
    \caption{
    A visualization of the construction of fair assignments $\bm{Z}^{\textup{test}}$ for the test data $\mathcal{D}^{(0), \textup{test}}$ and $\mathcal{D}^{(1), \textup{test}}$.
    The colored lines indicate the optimal transport maps.
    }
    \label{fig:K_nll_visual}
\end{figure}

We then draw a box-plot of $K$ versus the negative log-likelihood (\texttt{NLL}) over $K.$
It is observed that the test log-likelihood is smaller at $K=3$ (i.e., the posterior mode), which indicates that FBC is also good at inference of $K.$

\begin{figure}[h]
    \centering
    \includegraphics[width=0.5\linewidth]{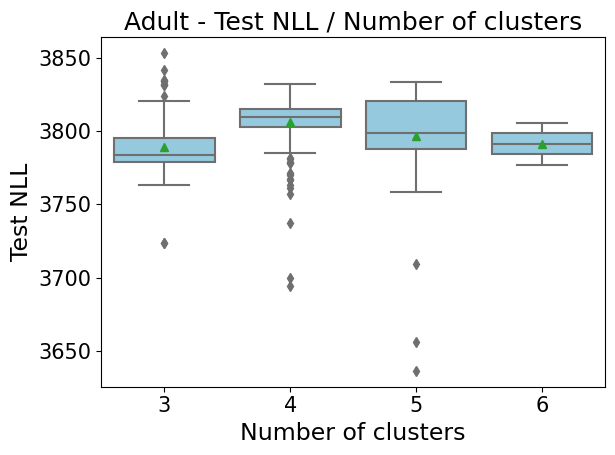}
    \caption{Number of clusters $K$ vs. Negative log-likelihood (\texttt{NLL}) on test data.}
    \label{fig:K_nll}
\end{figure}

\clearpage
\subsubsection{Application of FBC to categorical data}\label{sec:appen-exp-categorical}

\paragraph{Dataset construction}
We categorize the 6 continuous features (age, call duration, 3-month Euribor rate, number of employees, consumer price index, and number of contacts during the campaign) of \textsc{Bank} dataset by their median values.
For the sensitive attribute, we use the marital status (single/married) attribute, similar to the analysis of the original \textsc{Bank} dataset.

\paragraph{FBC for categorical data}
We consider the mixture of independent Bernoulli distributions, where each component $f(\cdot|\theta_k)$ is the product of Bernoulli distribution.
In other words, $f(\cdot|\theta_k) \sim \prod_{j=1}^{6} \textup{Bernouilli} (\cdot \vert p_{kj}),$
where $\theta_k=(p_{k1},\ldots,p_{k6})\in[0,1]^6$ and $p_{kj}\in[0,1]$ for $\forall j=1,\ldots,6$. 
We use the conjugate prior as $p_{kj}, j \in \{1, \ldots, 6\} \overset{\textup{i.i.d.}}{\sim}\textup{Beta}(\alpha,\alpha)$ with $\alpha=1$.

\paragraph{\texorpdfstring{\texttt{Cost}}{Cost} for categorical data}
For categorical data, we use the sum of negative log-likelihood values in each cluster as \texttt{Cost}.

\paragraph{Results}
The comparison results are provided in \cref{table:mixed}, which show that the fairness levels are significantly improved ($\Delta: 0.058 \to 0.004$ and $\texttt{Bal}: 0.200 \to 0.640$).
That is, FBC controls the fairness level well for categorical data. 
In addition, it is observed that the fairness constraint affects not only the fairness level but also the number of clusters. A similar behavior has been already observed in the analysis of the toy data in \cref{sec:toy}.


\begin{table}[h]
    \vskip -0.1in
    \footnotesize
    \caption{
    Comparison of inferred $K,$ \texttt{Cost}, $\Delta,$ and $\texttt{Bal}$ for MFM and FBC on the categorized \textsc{Bank} dataset.
    }
    \label{table:mixed}
    \centering
    \vskip 0.1in
    \begin{tabular}{c||c|c|c|c}
        \toprule
        Dataset & \multicolumn{4}{c}{Categorized \textsc{Bank}}
        \\
        \midrule
        Method & $K$ & \texttt{Cost} ($\downarrow$) & $\Delta$ ($\downarrow$)& $\texttt{Bal}$ ($\uparrow$)
        \\
        \midrule
        MFM & 5 & 2184.22 & 0.058 & 0.200
        \\
        FBC $\checkmark$ & 3 & 2907.31 & 0.004 & 0.640
        \\
        \bottomrule
    \end{tabular}
\end{table}


\clearpage
\subsubsection{Ablation studies}\label{sec:appen-exps-abl}

\paragraph{Temperature \texorpdfstring{$\tau$}{tau}}

\cref{table:appen-tau} reports the performance of FBC for different temperature values $\tau \in \{0.1,1.0,10.0 \}.$
Overall, varying $\tau$ does not affect much to the performance of FBC.

\begin{table}[h]
    \footnotesize
    \caption{
    Comparison of $K,$ \texttt{Cost}, $\Delta,$ and $\texttt{Bal}$ for $\tau \in \{ 0.1, 1.0, 10.0 \}.$
    }
    \label{table:appen-tau}
    \centering
    \vskip 0.1in
    \begin{tabular}{c||P{0.2cm}P{0.8cm}P{0.7cm}P{0.65cm}|P{0.2cm}P{0.8cm}P{0.7cm}P{0.65cm}|P{0.2cm}P{0.85cm}P{0.7cm}P{0.65cm}}
        \toprule
        \multirow{2}{*}{$\tau$} & \multicolumn{4}{c|}{\textsc{Diabetes}} & \multicolumn{4}{c|}{\textsc{Adult}} & \multicolumn{4}{c}{\textsc{Bank}}
        \\
        & $K$ & \texttt{Cost} ($\downarrow$) & $\Delta$ ($\downarrow$) & $\texttt{Bal}$ ($\uparrow$) & $K$ & \texttt{Cost} ($\downarrow$) & $\Delta$ ($\downarrow$) & $\texttt{Bal}$ ($\uparrow$) & $K$ & \texttt{Cost} ($\downarrow$) & $\Delta$ ($\downarrow$) & $\texttt{Bal}$ ($\uparrow$)
        \\
        \midrule
        0.1 & 4 & 6.776 & {0.007} & {0.917} & 3 & 4.991  & 0.003 & 0.433 & 7 & 5.947 & {0.020} & {0.500}
        \\
        1.0 & 4 & 6.795 & {0.012} & {0.910} & 3 & 4.989  & 0.006 & 0.429 & 7 & 5.947 & {0.020} & {0.500}
        \\
        10.0 & 4 & 6.795 & {0.012} & {0.910} & 3 & 4.989  & 0.006 & 0.429 & 7 & 5.947 & {0.020} & {0.500}
        \\
        \bottomrule
    \end{tabular}
\end{table}

\paragraph{Choice of \texorpdfstring{$R$}{R} when \texorpdfstring{$r > 0$}{r>0}}

In this section, we compare the two choices of $R$ in FBC:
(i) a random subset of $[n_0]$ and 
(ii) the index of samples closest to the cluster centers obtained by a certain clustering algorithm to $\mathcal{D}^{(0)}$ with $K=r,$
which is considered in \cref{sec:unequal_size}.
\cref{table:appen-heuristic} below provides the results, showing that the performance of the two approaches are not much different. 
Here, we utilized $K$-medoids algorithm to yield the cluster centers from $[n_0]$. 
Overall, we can conclude that FBC is not sensitive to the choice of $R$ and the two proposed heuristic approaches work well in practice.

\begin{table}[h]
    \footnotesize
    \caption{
    Comparison of $K,$ \texttt{Cost}, $\Delta,$ and $\texttt{Bal}$ for the two heuristic approaches in \cref{sec:unequal_size}.
    `Random' indicates the first approach ($R = \textup{ a random subset of } [n_0]$).
    and
    `Clustering' indicates the second approach ($R = \textup{ centers obtained by a clustering algorithm}$).
    }
    \label{table:appen-heuristic}
    \centering
    \vskip 0.1in
    \begin{tabular}{c||P{0.2cm}P{0.8cm}P{0.7cm}P{0.65cm}|P{0.2cm}P{0.8cm}P{0.7cm}P{0.65cm}|P{0.2cm}P{0.85cm}P{0.7cm}P{0.65cm}}
        \toprule
        \multirow{2}{*}{$R$} & \multicolumn{4}{c|}{\textsc{Diabetes}} & \multicolumn{4}{c|}{\textsc{Adult}} & \multicolumn{4}{c}{\textsc{Bank}}
        \\
        & $K$ & \texttt{Cost} ($\downarrow$) & $\Delta$ ($\downarrow$) & $\texttt{Bal}$ ($\uparrow$) & $K$ & \texttt{Cost} ($\downarrow$) & $\Delta$ ($\downarrow$) & $\texttt{Bal}$ ($\uparrow$) & $K$ & \texttt{Cost} ($\downarrow$) & $\Delta$ ($\downarrow$) & $\texttt{Bal}$ ($\uparrow$)
        \\
        \midrule
        Random & 4 & 6.795 & {0.012} & {0.910} & 3 & 4.989  & 0.006 & 0.429 & 7 & 5.947 & {0.020} & {0.500}
        \\
        Clustering & 4 & 6.719 & 0.007 & 0.892 & 3 &  4.986 & 0.012 & 0.375 & 5 & 5.988  & 0.040  & 0.549
        \\
        \bottomrule
    \end{tabular}
\end{table}

\clearpage
\paragraph{Prior parameter \texorpdfstring{$\kappa$ in $p_{K}$}{kappa in p_K}}

We compare the posteriors of $K$ for various values of $\kappa$ in $\textup{Geometric}(\kappa).$
\cref{fig:kappa_abl} below shows the results, suggesting that FBC is not sensitive to the choice of $\kappa.$
For example, the posterior mode of $K$ is consistently $3$ for \textsc{Adult} dataset.

\begin{figure}[h]
    \centering
    \includegraphics[width=0.28\linewidth]{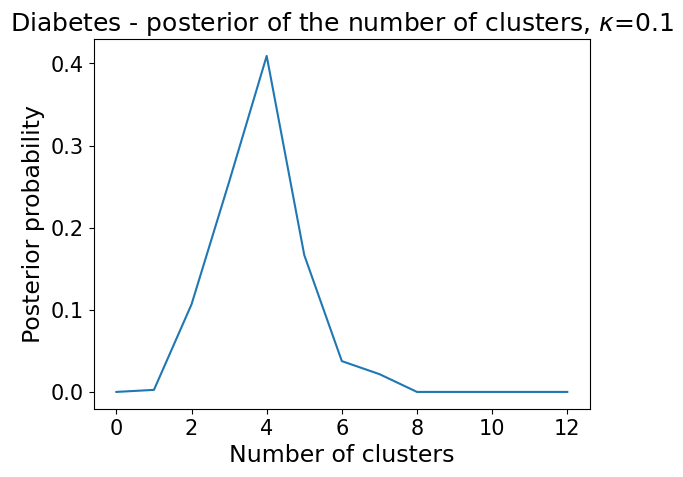}
    \includegraphics[width=0.27\linewidth]{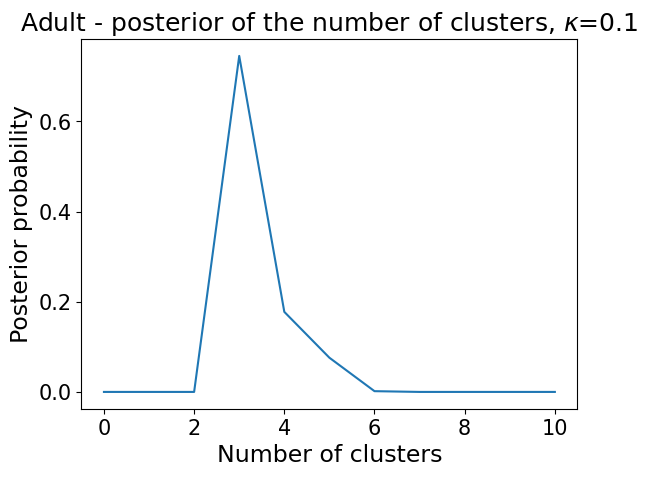}
    \includegraphics[width=0.27\linewidth]{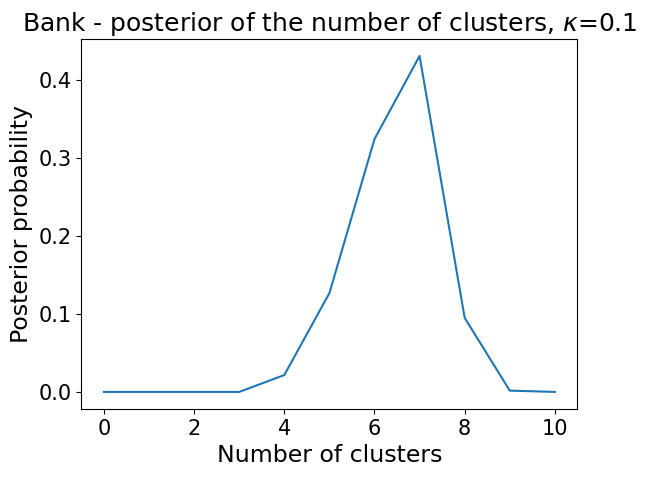}
    \\
    \includegraphics[width=0.28\linewidth]{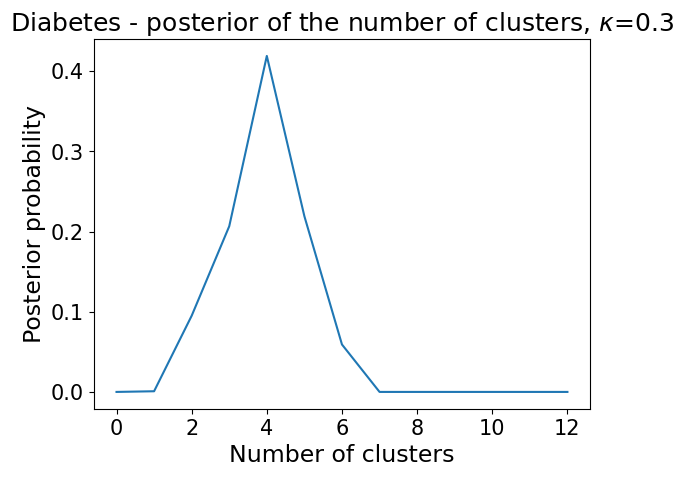}
    \includegraphics[width=0.27\linewidth]{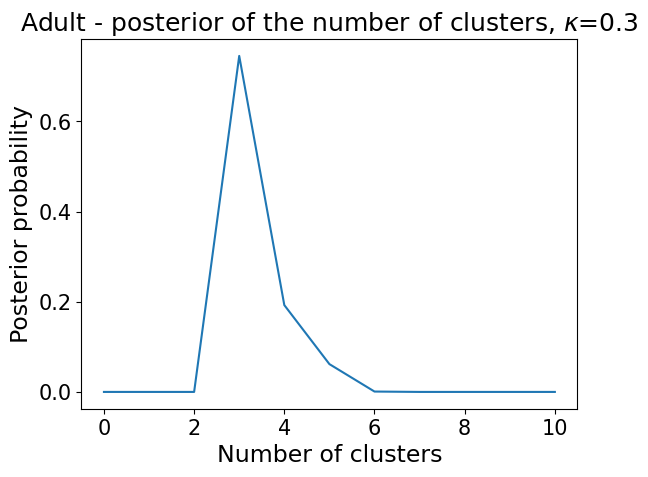}
    \includegraphics[width=0.27\linewidth]{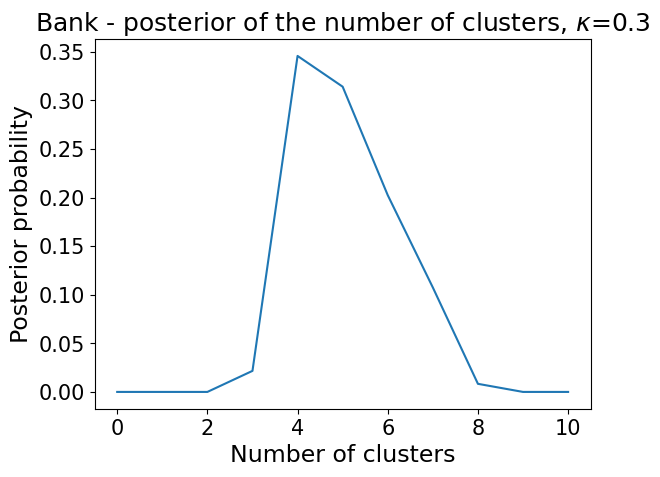}
    \\
    \includegraphics[width=0.28\linewidth]{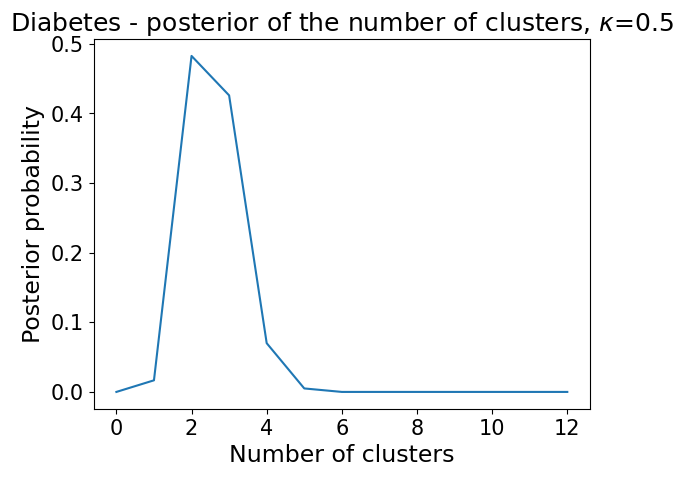}
    \includegraphics[width=0.27\linewidth]{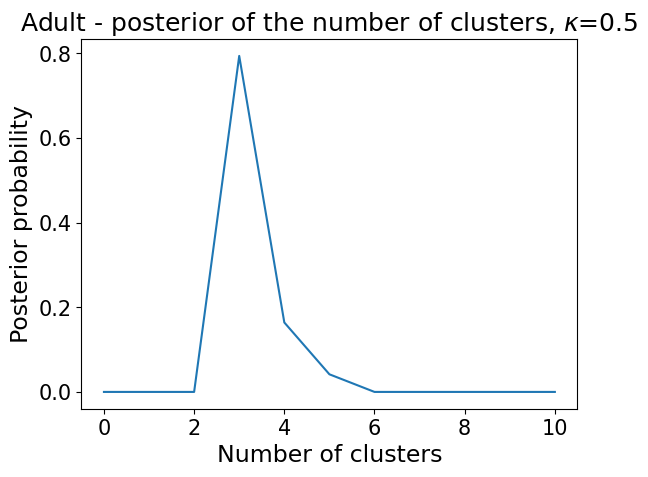}
    \includegraphics[width=0.27\linewidth]{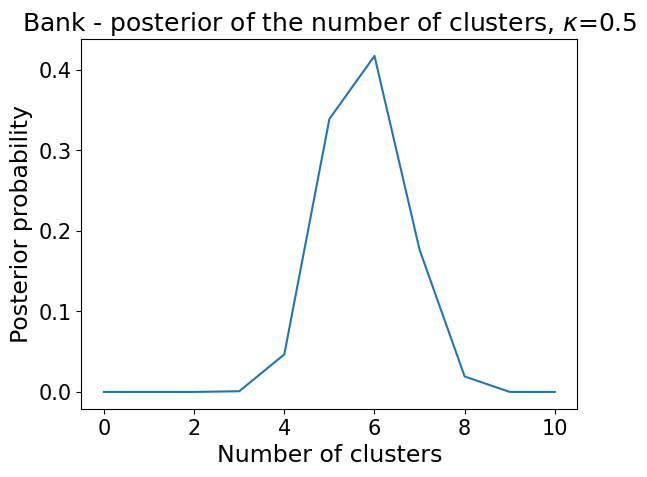}
    \\
    \includegraphics[width=0.28\linewidth]{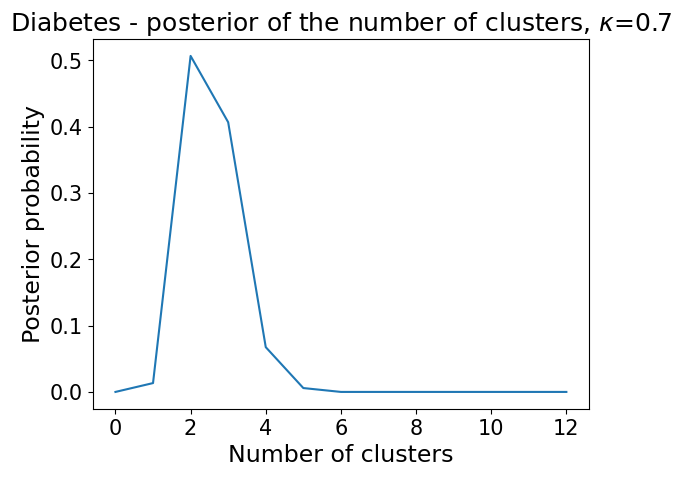}
    \includegraphics[width=0.27\linewidth]{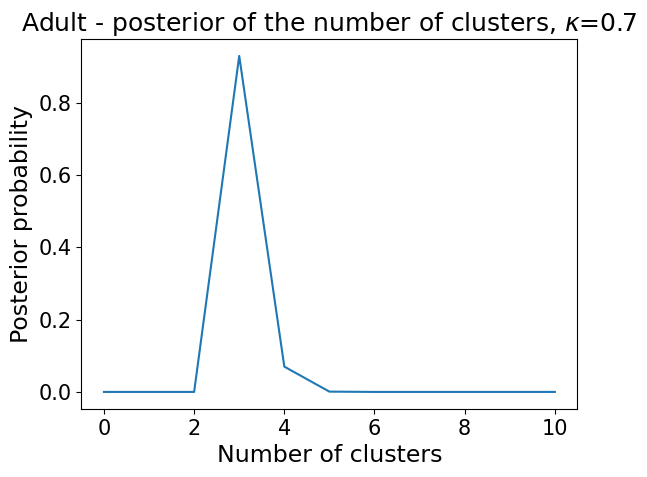}
    \includegraphics[width=0.27\linewidth]{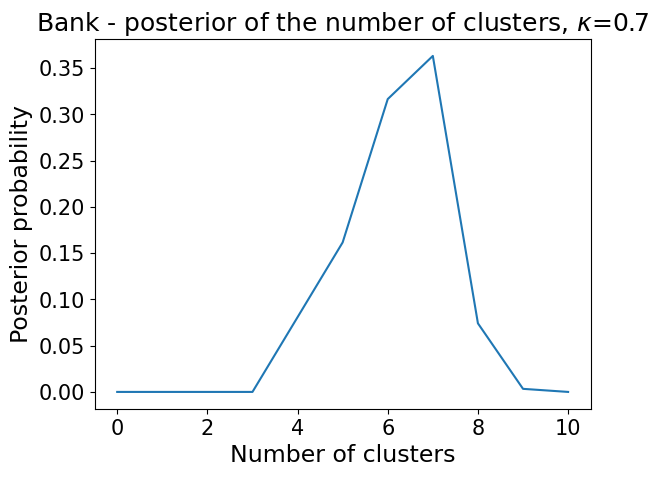}
    \\
    \includegraphics[width=0.28\linewidth]{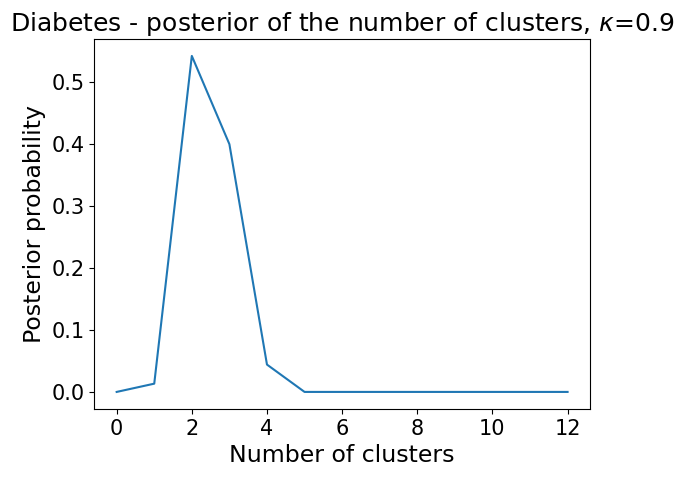}
    \includegraphics[width=0.27\linewidth]{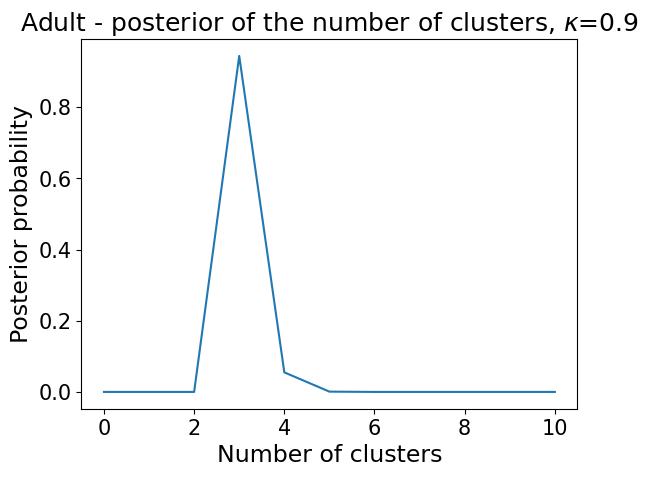}
    \includegraphics[width=0.27\linewidth]{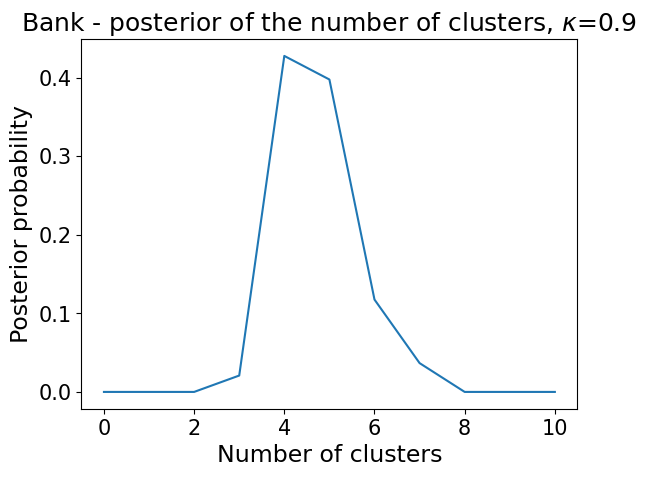}
    \caption{Posteriors of $K$ on (left) \textsc{Diabetes}, (center) \textsc{Adult}, and (right) \textsc{Bank} datasets, for $\kappa \in \{ 0.1, 0.3, 0.5, 0.7, 0.9 \}$ (from top to bottom).}
    \label{fig:kappa_abl}
\end{figure}


\clearpage
\paragraph{Convergence of MCMC}

We show that the proposed MCMC algorithm converges well in practice.
To do so, we analyze the autocorrelation function of the inferred $K$ and the negative log-likelihood (\texttt{NLL}) on training data (i.e., the observed instances).

\begin{enumerate} 
    \item[$K$:] As is done by \cite{miller2018mixture}, in Figure \ref{fig:conv_abl},
    we draw the autocorrelation function $\rho(h)$ defined as
    $$\rho(h)={\rm corr} \{ (K_t, K_{t+h}), t=1,\ldots \},$$
    where $K_t$ is a posterior sample at iteration $t.$
    The autocorrelation functions amply support that the proposed MCMC algorithm converges well.

    \begin{figure}[h]
        \centering
        \includegraphics[width=0.32\linewidth]{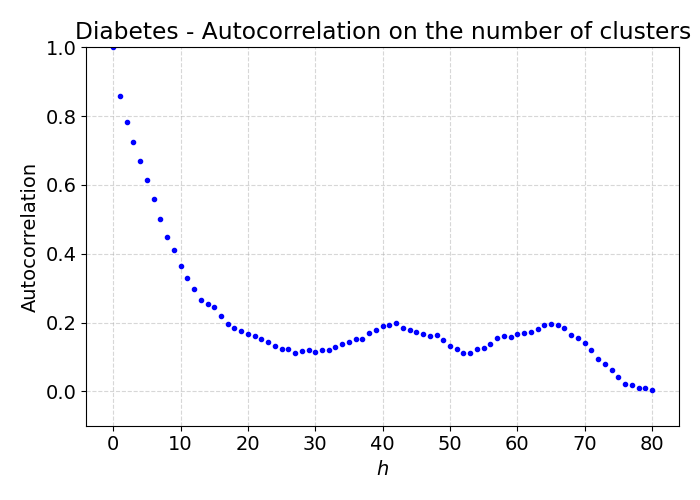}
        \includegraphics[width=0.32\linewidth]{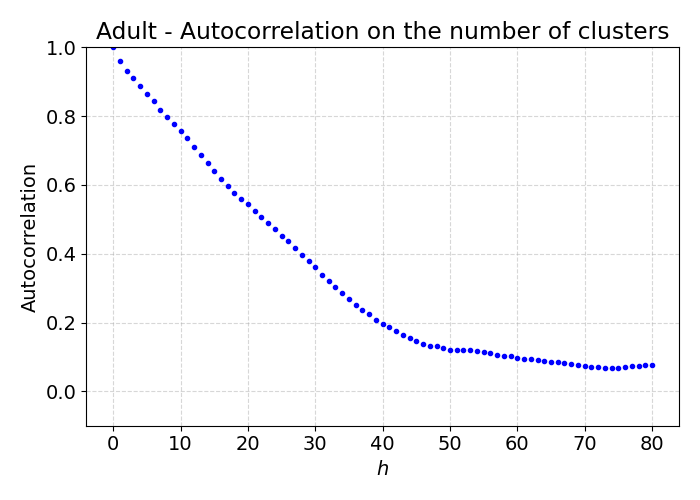}
        \includegraphics[width=0.32\linewidth]{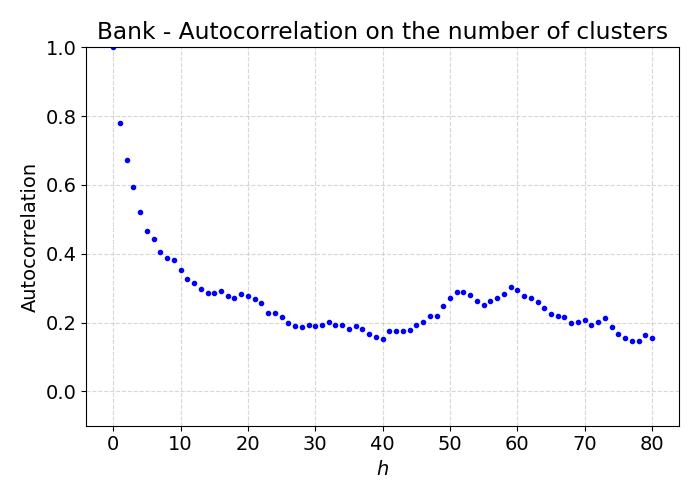}
        \caption{$h$ vs. Autocorrelation functions for (left) \textsc{Diabetes}, (center) \textsc{Adult}, and (right) \textsc{Bank} datasets.}
        \label{fig:conv_abl}
    \end{figure}

    \item[\texttt{NLL}:]
     Figure \ref{fig:conv_abl_nll} draws the trace plots of the \texttt{NLL} on training data.
     Dramatic decreases of \texttt{NLL} are observed which would happen when the MCMC algorithm moves one local optimum to another local optimum.
     Fortunately, the new optima are always better than the old ones, supporting the ability of FBC to explore efficiently for searching good clusters.

    \begin{figure}[h]
        \centering
        \includegraphics[width=0.32\linewidth]{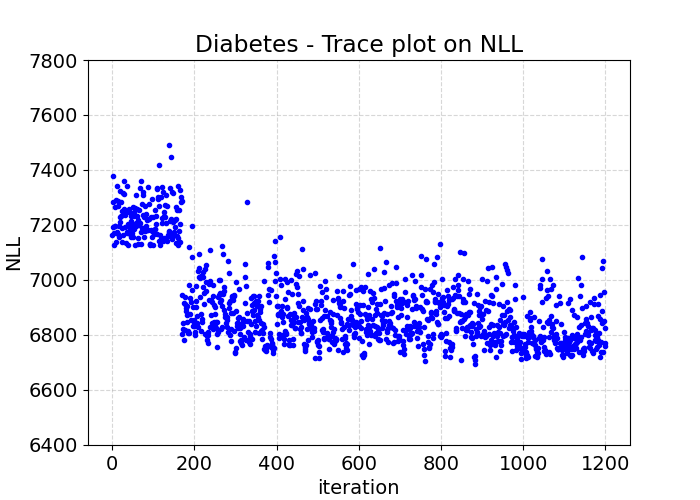}
        \includegraphics[width=0.32\linewidth]{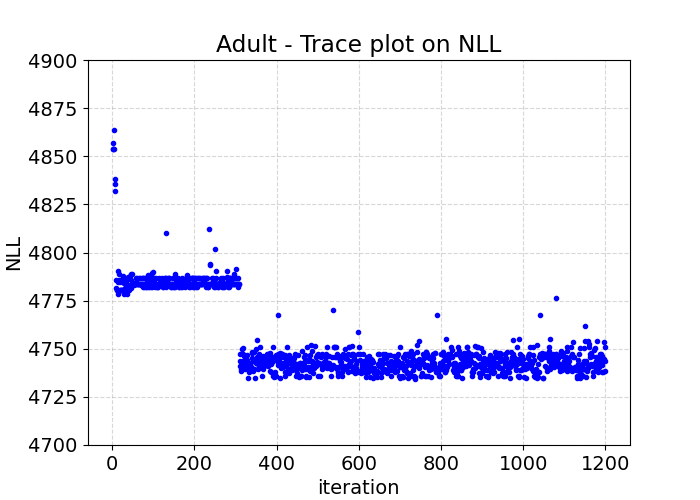}
        \includegraphics[width=0.32\linewidth]{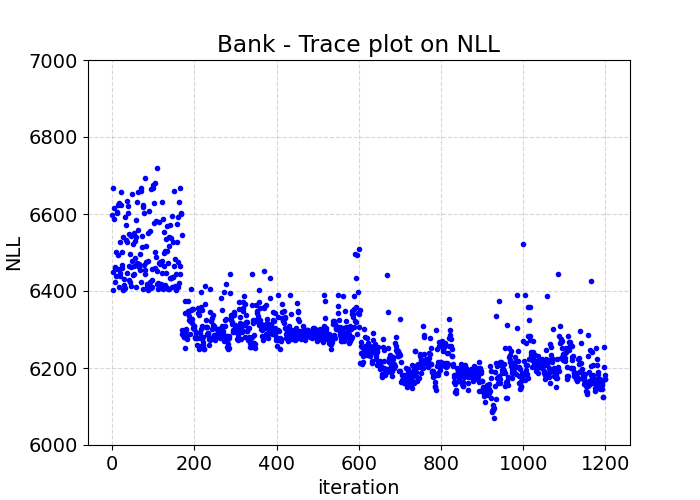}
        \caption{Trace plots of $\texttt{NLL}$ on (left) \textsc{Diabetes}, (center) \textsc{Adult}, and (right) \textsc{Bank} datasets.}
        \label{fig:conv_abl_nll}
    \end{figure}
\end{enumerate}



\end{document}